%% file: 0_main.tex
\documentclass[twoside]{article}

\usepackage[preprint]{aistats2025}

\usepackage{my_macros}

\usepackage[round]{natbib}

\usepackage{color}
\usepackage[vlined]{algorithm2e}
\SetKwInput{Input}{Input}
\SetKwInput{Init}{Initialization}

\usepackage{amsmath,amssymb}
\usepackage{amsthm}
\usepackage{bbm}
\usepackage[normalem]{ulem}
\usepackage{enumitem}
\usepackage[many]{tcolorbox}
\usepackage[capitalize,noabbrev]{cleveref}
\usepackage{graphicx}  % for including images
\usepackage{subfig} % for subfigures
\usepackage{float}     % for positioning the figure
\usepackage{xcolor}
\usepackage{multicol}
\usepackage{outlines}
\usepackage{thmtools,thm-restate}
\usepackage[export]{adjustbox}  % For valign option

\usepackage{tikz}
\usetikzlibrary{calc}

\newtcolorbox{cross}{blank,breakable,parbox=false,
    overlay={\draw[red,line width=5pt] (interior.south west)--(interior.north east);
    \draw[red,line width=5pt] (interior.north west)--(interior.south east);}}

\newtheorem{lemma}{Lemma}
\newtheorem{theorem}{Theorem}
\newtheorem{cor}{Corollary}
\newtheorem{remark}{Remark}

\begin{document}

    \twocolumn[

        \aistatstitle{Distributed Online Optimization with Stochastic Agent Availability}

        \aistatsauthor{ Juliette Achddou \And Nicolò Cesa-Bianchi \And  Hao Qiu }

        \aistatsaddress{ CRIStAL, Université de Lille, Inria, \\ CNRS, Centrale Lille\\
        Lille, France \And  Università degli Studi di Milano 
   \\ and Politecnico di Milano\\Milan, Italy \\ \\
    \And Università degli Studi di Milano, \\Milan, Italy\\ } 
]

\begin{abstract}
% In distributed online optimization, multiple agents sequentially cooperate to optimize a global objective by communicating through a fixed network.
Motivated by practical federated learning settings where clients may not be always available, we investigate a variant of distributed online optimization where agents are active with a known probability $p$ at each time step, and communication between neighboring agents can only take place if they are both active.
We introduce a distributed variant of the FTRL algorithm and analyze its network regret, defined through the average of the instantaneous regret of the active agents.
Our analysis shows that, for any connected communication graph $G$ over $N$ agents, the expected network regret of our FTRL variant after $T$ steps is at most of order $(\sym/p^2)\min\big\{\sqrt{N},N^{1/4}/\sqrt{p}\big\}\sqrt{T}$, where $\sym$ is the condition number of the Laplacian of $G$.
% Moreover, if the edges in $G$ are stochastically active with probability $q$ at each step, we show that the above regret gets multiplied by $\frac{1}{q}$. 
We then show that similar regret bounds also hold with high probability.
Moreover, we show that our notion of regret (average-case over the agents) is essentially equivalent to the standard notion of regret (worst-case over agents), implying that our bounds are not significantly improvable when $p=1$.
Our theoretical results are supported by experiments on synthetic datasets.
\end{abstract}

    \input{1-intro.tex}
    \input{2-gen_results.tex}

    \input{1b-complexity}
    \input{3-spectral-prop.tex}
    \input{4-erdosrenyi}
    \input{5-exp}

    \input{6-limitations}

\subsubsection*{Acknowledgements}
This work was done when JA was research assistant at Università degli Studi di Milano, Italy.
All the authors acknowledge the financial support from the FAIR (Future Artificial Intelligence Research) project, funded by the NextGenerationEU program within the PNRR-PE-AI scheme (M4C2, investment 1.3, line on Artificial Intelligence) and the One Health Action Hub (1H-Hub) within the PSR-LINEA6 (MUR, DM 737/2021). JA and NCB also acknowledge the financial support from the EU Horizon CL4-2021-HUMAN-01 research and innovation action under grant agreement 101070617, project ELSA (European Lighthouse on Secure and Safe AI). 

    \bibliography{ref.bib}
    \clearpage
    \appendix
    \onecolumn
    \input{10_Appendix}

\end{document}

%% file: 1-intro.tex
\section{\uppercase{Introduction}}
\label{sec:intro}
% robotics, sensor networks, networks of mobile phones etc \citep{lithuan20, halsted2021surveydistributedoptimizationmethods, hard2019federatedlearningmobilekeyboard, rabbat2004distributed}. %\cite{journals/tkde/YanSVQ13,conf/cdc/HosseiniCM13}，).
Distributed convex optimization \citep{nedic2009distributed,duchi2011dual} is an algorithmic framework widely used in federated learning settings---see, e.g., \citet{yang2019survey} for a survey of applications.
In distributed convex optimization the goal is to optimize a global convex objective by means of a decentralized algorithm run by multiple agents who can exchange messages only with their neighbors in a given communication network. The global objective is expressed as a sum of local functions associated with the agents, where each agent has only oracle access to its local function (typically, via a first-order oracle).

In this work we focus on distributed online optimization (DOO), an online learning variant of distributed convex optimization in which each agent is facing an adversarial sequence of convex loss functions \citep{hosseini2013online}. The goal of an agent is to minimize its regret with respect to a sequence of global loss functions, each obtained by summing the corresponding local losses for each agent. In both batch and online distributed optimization settings, the presence of the communication network, which limits the exchange of information to adjacent nodes, implies that agents must use some information-propagation technique to collect information about the global loss function.  

Our main contribution in this work is the analysis of a variant of DOO in which agents may not be available in every time step. When inactive, an agent neither contributes to the regret nor can it communicate with its neighbors. In practice, this scenario may arise due to machine failures, disconnections, or devices (e.g., mobile phones) being turned off. While the problem of intermittent agents availability has been investigated before in distributed convex optimization \citep{gu2021fast,wang2022unified,yan2024federated}, we are not aware of any such study in the DOO framework.
More specifically, we consider random agent activations where, in each round, each agent $v$ becomes independently active with some unknown probability $p_v$. As a consequence, the active communication network at time $t$ becomes stochastic, as it is induced by the random subset of active agents at time $t$.
Because the number of active agents is a random variable, to ensure a uniform scaling of losses across times steps we define the global loss function as an average (as opposed to a sum) over the active agents. Likewise, we define the instantaneous regret as an average over the active agents. We call network regret the sum of these instantaneous regrets. Although network regret may seem weaker than the notion of regret commonly used in DOO settings where agents are always active, we show that the two notions are essentially equivalent, see \Cref{rem:doo-tight}.

\cite{hosseini2016online,lei2020online} investigated DOO on random communication networks. However, these previous studies only consider stochastic edge availability. Our analysis of networks with random node availability provides a set of results that can be applied to both edge and node availability. In particular, we recover the bounds of \citep{lei2020online} for the full information setting as a special case of ours. 

To propagate information, we use standard gossiping techniques \citep{journals/scl/XiaoB04,boyd2006randomized}, that average gradient information from neighboring agents. Alternatively, one can propagate information about local losses by message passing. This approach has two main drawbacks. When the network is dense or the activation probabilities are high, the number of messages becomes too big. Conversely, when the network is sparse or the activation probabilities are small, the time it takes for an agent to collect all gradient information for a given round becomes too long. \looseness -1

\paragraph{Main contributions.}
The main contributions of this work can be summarized as follows.
\begin{itemize}[topsep=0pt,parsep=0pt,wide]
%    \item We define a notion of network regret tailored to a setting of intermittent agent availability, where the instantaneous regret is averaged over the agents that are active in that round.
    \item We design and analyze \algoname, a distributed variant of the FTRL algorithm for online convex optimization, and prove a general network regret bound for arbitrary connected communication networks $G$ and arbitrary activation probabilities.
    \item In the $p$-uniform case (activation probabilities equal to some known $p \ge 1/N$), the expected regret of our algorithm is bounded by $\frac{1}{1-\rho}\min\big(\sqrt{N},\frac{N^{1/4}}{\sqrt{p}}\big)\sqrt{T}$, where $T$ is the known time horizon, $N$ is the known number of agents, and $1-\rho$ is the unknown spectral gap of the gossip matrix supported on $G$. For the same algorithm, we also prove a regret bound of order $\frac{N}{1-\rho^2}\sqrt{T}$ that holds with high probability for any arbitrary activation probabilities.
    \item For a standard choice of the gossip matrix, we show that $\frac{1}{1-\rho}$ is of order ${\sym}/{p^2}$ in the $p$-uniform case, where $\sym$ is the condition number of the Laplacian matrix of $G$.
    If the spectral radius of $G$ is known, then the network regret is bounded by $\frac{\sym(G)}{p}\min\big(\sqrt{N},\frac{N^{1/4}}{\sqrt p}\big)\sqrt{T}$.
    % If edges between active agents can become inactive with probability $1-q$ in each round---as in \citep{lei2020online}---we find that $\frac{1}{1-\rho}$ becomes of order $\sym/(p^2q)$.
    \item For $p=1$ (all agents are always active), we prove a lower bound showing that any distributed online algorithm must suffer, on some $G$, a network regret at least of order $\big(\frac{\rho}{1-\rho}\big)^{\alpha/4}N^{(1-\alpha)/2}\sqrt{T}$ for any $0 \le \alpha \le 1$. By comparison, when tuned for $\rho$, our regret bound is of order $\big(\frac{1}{1-\rho}\big)^{1/2}N^{1/4}\sqrt{T}$.
    \item Finally, we run experiments on synthetic data comparing \algoname\ with \dogd\ by \citet{lei2020online} for different choices of the relevant parameters. % of $G$, $p$, and $q$.
    % \footnote{In the setting of \citet{lei2020online}, $1-q$ denotes the probability that an edge between active agents becomes inactive in any given round.}
\end{itemize}
%
% Unless specified otherwise, our bounds hold when agents only know $N$ and $T$. In particular, agents need not know the structure of $G$.
%
Our most general bounds (\Cref{th:lin_thm}) hold when agents only know $p$, $N$ and $T$. In particular, agents need not know the structure of $G$ and \algoname\ is run with the same initialization for all agents. The more refined bounds in \Cref{cor:scaling} and \Cref{erdosrenyi} need also the preliminary knowledge of the spectral radius of $G$ (or of a suitable bound on it).

\paragraph{Technical challenges.}
The fact that our global loss at time $t$ is averaged over the random subset of active agents requires some substantial modifications to previous analyses of DOO.
Unlike previous works, which use Online Gradient Descent and Dual Averaging as base algorithms, we prove our results using the general FTRL algorithm with arbitrary regularizers.
 
To prove our main result (\Cref{th:lin_thm}), we introduce an omniscient agent, who runs FTRL knowing the past loss gradients of all active agents. We then decompose the regret in two terms: one measuring the regret of the omniscient agent, which we can bound using single-agent FTRL analysis, and one accounting for the deviations between the predictions of the omniscient agent and of the individual agents.
To analyze the latter term, we first apply convex analysis to bound the norm of the difference between the prediction  $\xbar_t$ of the omniscient agent and the prediction $x_t(v)$ of each individual agent $v$ in terms of the dual norm of the cumulative gradients of the corresponding global losses, $\zbar_t$ and $z_t(v)$, where $z_t(v)$ is the gossiping-based local estimate of $\zbar_t$.
%
%A technical hurdle, missing in previous analyses, concerns 
We then relate the expected value of these dual norms, averaged over the random set of active agents, with terms of the form $\E\big\|  \prod_s W_s - \frac{\bone\bone^{\top}}{N} \big\|_{2}$, which appear in the analysis of gossip algorithms. Here $W_s$ are doubly stochastic gossip matrices and $\bone = (1,\ldots,1)$. A technical hurdle, missing in previous analyses, is to account for the random number of active agents when relating the two quantities above and also when bounding the regret of the omniscient agent.
% (because the result obtained by using single-agent FTRL analysis concerns a different notion of regret involving the sum of active agents' losses). 
This is due to the fact that a direct application of single-agent FTRL analysis only bounds the sum of active agents' losses rather than their average. 

\section{\uppercase{Setting and notation}}
\label{sec:setting}
In  multi-agent online convex optimization, agents are nodes of a communication network represented by a connected and an undirected graph $\scG = (V,E)$, where $V = \{1,\ldots,N\} = [N]$  indexes the agents and the edge set $E$ defines the communication structure among agents. We use $\sN_v=\{u \in V \,:\, (u,v) \in E\}$ to denote the neighborhood of $v\in V$.
Let $\mathcal{X} \subset \mathbb{R}^d$ be the agents' common decision space, which we assume to be convex and closed. A arbitrary and unknown sequence $\ell_1(v,\cdot),\ell_2(v,\cdot),\ldots$ of \textit{local losses} $\ell_t(v,\cdot) : \mathcal{X} \to \mathbb{R}$ is associated with each agent $v\in\ V$. For all $t\ge 1$ we assume $\ell_t(v, \cdot)$ is convex and $L$-Lipschitz with respect to an arbitrary norm $\|\cdot\|$. At every round $t=1,2,\ldots$, each agent $v\in V$ becomes independently active with fixed probability $p_v$. Without loss of generality, we assume $\pmin = \min_v p_v > 0$ and, for simplicity, $\sum_v p_v \ge 1$ (otherwise less than one active agent would be active per round on average). Note that this implies $\pmax = \max_v p_v \ge \frac{1}{N}$. We call \textit{$p$-uniform} the special case when $p_v = p$ for all $v \in V$. \looseness -1
We assume that active agents $v$ know which of their neighbors in $\sN_v$ are active.
Let $S_t$ be the set of active agents at time $t$ and $E_t = E \cap \{(u,v) \,:\, u,v \in S_t\}$ be the set of active edges at time $t$, i.e., edges in $E$ whose both endpoints are active in that round.
We say that a doubly stochastic matrix $W_t$ is a gossip matrix for the set $S_t$ of active agents at time $t$ if $W_t(v,v') = 0$ for all distinct $v,v' \in V$ such that $(v,v') \not\in E_t$. Let $W$ be a random gossip matrix (with respect to a graph $G$ and activation probabilities $p_v$ for $v \in V$) and define $\rho =\sqrt{\mylambdatwo}$,  for i.i.d gossip matrices $W_t$, where we denote by $\lambda_i(\cdot)$ the $i$-th highest eigenvalue of a matrix (keeping track of multiplicity). Clearly, $0 < \rho < 1$. In what follows, we often write $\rho$ leaving $G$ and $\{p_v\}_{v\in V}$ implicitly understood from the context.

Next, we define the distributed online optimization protocol used in this work.
\\
At each round $t=1,2,\ldots,T$,
\begin{enumerate}[topsep=0pt,parsep=0pt,itemsep=1pt]
    \item Each active agent $v \in S_t$ chooses an action $x = x_t(v) \in \mathcal{X}$ and observes the gradient $\nabla\ell_t(v,x)$ of the local loss $\ell_t(v,\cdot)$.
%The agent can not observe the losses of other agents. 
    \item Each active agent $v \in S_t$ sends a message $z_t(v)$ to their active neighbors and uses the messages received from the active neighbors to compute a new message $z_{t+1}(v)$.
\end{enumerate}
Note that this protocol implicitly defines an active communication graph $G_t = (S_t,E_t)$ for round $t$.
As the number of agents that incur loss in a step is a random variable, we define the \textit{network loss} at step $t$ as the average over the active agents of the local losses in that round,
\\[2mm]
\centerline{ ${\displaystyle
    \ellbar_t(S_t,\cdot) = \frac{1}{|S_t|} \sum_{v \in S_t} \ell_t(v,\cdot)
}$ }
\\[2mm]
and let $\ellbar_t(\emptyset,\cdot) = 0$. Hence, unlike the standard DOO model where $\ellbar_t$ scales linearly with $N$, in our model $\ellbar_t$ is independent of $N$.

The agents' performance is measured by the \textit{network regret} $\Regbar_T$ defined by
\\[2mm]
\centerline{ ${\displaystyle
    \sum_{t=1}^T \frac{1}{|S_t|} \sum_{v\in S_t} \ellbar_t\big(S_t,x_t(v)\big)
    -
    \min_{x \in \sX} \sum_{t=1}^T \ellbar_t(S_t,x)
    %\tag{\textbf{R1}}
}$ }
\\[2mm]
where
% \[
%     r_t (v) = \ellbar_t\big(S_t,x_t(v)\big)
% -
%     \min_{x \in \sX}\sum_{t=1}^T \ellbar_t(S_t,x)\]
% and
%we define $\frac{1}{|S_t|} \sum_{v\in S_t} r_t(v)=0$
the steps $t$ when $S_t = \emptyset$ are omitted from the sum.
In this work, we provide bounds in high probability and in expectation for the network regret.
\begin{remark}
\label{rem:doo-tight}
When $p = 1$, we recover the standard DOO setting and our network regret $\Regbar_T$ becomes
\begin{align}
\label{eq:net-reg-p1}
    \frac{1}{N} \sum_{v\in V} \sum_{t=1}^T \ellbar_t\big(V,x_t(v)\big)
-
    \min_{x \in \sX}\sum_{t=1}^T \ellbar_t(V,x)~.
\end{align}
In standard DOO, $\ellbar_t$ is a sum over the $N$ local losses, and $\frac{1}{N} \sum_{v\in V}$ is replaced by $\max_{v \in V}$. The resulting regret $R_T$ is then defined by
\begin{equation}
\label{eq:doo-regret}
    \max_{v\in V} \sum_{t=1}^T \sum_{v'\in V} \ell_t\big(v',x_t(v)\big)
-
    \min_{x \in \sX}\sum_{t=1}^T \sum_{v'\in V} \ell_t(v',x)~.
\end{equation}
Hence, when $p=1$, $\Regbar_T \le R_T/N$. Recently, \citet{pmlr-v247-wan24a} proved that $R_T = \tilde{\Theta}\big(N(1-\rho)^{-1/4}\sqrt{T}\big)$ where the upper bound relies on accelerated gossiping and $\tilde{\Theta}$ hides factors logarithmic in $N$. Hence, using the same algorithm, we get $\Regbar_T = \tilde{\scO}\big((1-\rho)^{-1/4}\sqrt{T}\big)$. In \Cref{th:lb} we prove that $\Regbar_T = \Omega\big((1-\rho)^{-1/4}\sqrt{T}\big)$ for any distributed online algorithm, thus proving that $\Regbar_T = \tilde{\Theta}(R_T)/N$ in the special case $p=1$. 
\end{remark}
%
%In the setting of online convex optimization with heterogeneous losses \citep{hosseini2016online}, each agent $v$ is always active but the network's topology changes over time. The performance of any agent $v$ is measured by
%
%\begin{equation*}
%    R_T
%=
%    \sum_{t=1}^T \ellbar_t\big(V,x_t(v)\big)
%-
%    \min_{x \in \sX}\sum_{t=1}^T \ellbar_t(V,x), \tag{\textbf{R3}}
%\end{equation*}
%where $\ellbar_t(V,\cdot)=\frac{1}{|V|} \sum_{v \in V} \ell_t(V,\cdot)$.
%
%We can also measure the performance of all agents by \textbf{R2}.
%
%Note that 
%\begin{equation*}
%    \textbf{R2} \leq \textbf{R3}.
%\end{equation*}
%
\begin{remark}
    In the multi-agent single-task setting of \citet{cesa2020cooperative}, local losses are the same for each agent, $\ell_t(v,\cdot) = \ell_t(\cdot)$ for all $v\in V$. The network regret then takes the form
    \begin{align*}
        \Regbar_t
    =
        \sum_{t=1}^T \frac{1}{|S_t|} \sum_{v\in S_t} \ell_t\big(x_t(v)\big)
    -
        \min_{x \in \sX} \sum_{t=1}^T \ell_t(x)~.
    \end{align*}
    This setting is not comparable with DOO. Indeed, in the single-task setting agents can achieve an expected network regret of order $\mathcal{O}(\sqrt{T})$ even without communicating. In DOO, instead, ignoring communication leads to a linear expected network regret.
\end{remark}

%% file: 2-gen_results.tex
\section{\uppercase{The \protect{\algoname}\ algorithm}}

%In this section, we start by examining, as a warm-up, the scenario where local loss functions are convex, but the number of active agents $S_t$ is known by every agent  at each round $t$.

%\subsection{Algorithm}
%
We assume each agent runs an instance of \algoname\ (\Cref{alg:known}), a gossiping variant of FTRL with a regularizer $\psi: \mathcal{X} \mapsto \R$ that is $\mu$-strongly convex with respect to the same norm $\|\cdot\|$ with respect to which the Lipschitzness of the losses is defined. Our analysis depends on the choice of $\psi$ only through $\mu$ and the diameter $D^2 = \max_{x\in\sX} \psi(x) - \min_{x'\in\sX} \psi(x')$.
At any time step $t$, the instance of \algoname\ run by an active agent $v$ computes a weight vector $W_t(v,\cdot)$ over the set $\scN_v \cap S_t$ of active neighbors.
% our algorithm requires access to a doubly-stochastic and symmetric \textit{weight matrix} $W_t$, only supported on active edges.
%
% Each agent $v\in V$ only needs to access the $v$-th row of the weight matrix. 
%
In \Cref{sec:lambdatwo}, we introduce a simple way of choosing these weights so that the \textit{gossip matrix} $W_t(\cdot,\cdot)$ is a doubly stochastic matrix, which is a requirement for our analysis.
% such that any agent can recover their row by only knowing their neighbors. %, but for now, we give results valid for any choice of $W_t$. 
%
% For the analysis, we need the $W_t$ to be be i.i.d. 
%

%
% Each active agent's initial prediction $x_1(v)$ at time $1$ is set to $\argmin_{x \in \mathcal{X}} \psi(x)$.
\begin{algorithm}[t]
    \caption{\algoname. An instance of this algorithm is run by each agent $v \in V$}
    \label{alg:known}
    \Input{Learning rate $\eta > 0$}

    \Init{$z_1(v) = 0$}

    \For{$t=1,2,\ldots$}{

        \If{$v \in S_t$}{
            Predict ${\displaystyle x_t(v) = \argmin_{x \in \mathcal{X}}\left\{\langle z_t(v), x\rangle+\frac{1}{\eta} \psi(x)\right\} }$

            Observe $g_t(v) = \nabla\loss_{t}\big(v,x_t(v)\big)$

            Send $z_t(v)$ to $\scN_v \cap S_t$
            
            Receive and store $z_t(j)$ from $j \in \scN_v \cap S_t$

            Compute $W_t(v,j) > 0$ for $j \in \scN_v \cap S_t$

            Compute ${\displaystyle z_{t+1}(v) = \sum_{j \in \scN_v \cap S_t} W_t(v,j) z_t(j) + g_t(v) }$

        }
        \Else{
            $z_{t+1}(v) = z_t(v)$
        }
    }

\end{algorithm}

The instance of \algoname\ run by agent $v$ has two local variables: $g_t(v)$, corresponding to the local loss gradient for the prediction $x_t(v)$ of $v$ at time $t$, and $z_t(v)$, corresponding to the estimate of the network loss gradient computed by agent $v$.

Let $\bg_t \in \R^{N\times d}$ be the matrix whose $v$-th row is
\begin{equation}
    \label{eq:gr}
    g_t(v) = \left\{
    \begin{aligned}
        & 0 & \text{if $v \notin S_t$,}
        \\
        &  \nabla\loss_{t}\big(v,x_t(v)\big) & \text{if $v \in S_t$.}
    \end{aligned}
    \right.
\end{equation}
Correspondingly, we define $\bz_t$, the matrix whose $v$-th row is $z_t(v)$ for all $v\in V$.
Let $\be_v$ be the canonical basis vector for coordinate $v\in [N]$.
Let the weights $W_t(v,\cdot)$ computed by the instance of \Cref{alg:known} run by each $v\in S_t$ form a $N \times N$ gossip matrix $W_t$ for $S_t$ such that $W_t(v,\cdot) = \be_v$ for all $v\in V\setminus S_t$. We may write the updates performed by the instance as
$
    \bz_{t+1} = W_t \bz_t + \bg_t
$.
Note that the definitions of $W_t$ and $\bg_t$ imply that $z_{t+1}(v) = z_t(v)$ for all agents $v\in V\setminus S_t$ that are inactive at time $t$. Moreover, any active agent $v \in S_t$ can compute $z_{t+1}(v)$ using the most recent value $z_s(j)$ (for some $s < t$) received by agents $j \in \scN_v \setminus S_t$ (i.e., neighbors inactive at time $t$).

The algorithm considered here is a natural extension to arbitrary regularizers and random activations of the algorithms traditionally used for DOO, e.g., \citep{hosseini2013online,lei2020online}.
The optimal choice for the learning rate is however different, as it depends on the activation probabilities.
%
% Another distinction with \cite{lei2020online} is the use of FTRL instead of online gradient descent, making our algorithm more versatile, and in particular fit for the experts setting. 

\section{\uppercase{Upper bounds}}
\label{sec:expresults}
% \begin{assumption}
% We assume that
% %
% \begin{itemize}
% %
% \item $\forall x \in \mathcal{X}, ~\psi(x) \leq D^2$,
% %
% \item $\psi$ is $\mu$-strongly convex with respect to a norm $\|\cdot\|$,
% %

% \end{itemize}
% \end{assumption}
%
Recall that at each round $t$, each agent $v \in V$ is independently active with probability $p_v > 0$. Let $\pbar$ the average of these probabilities and $\sigma_p^2 = \frac 1 N \sum_{v \in V} p_v^2 - \pbar^2$ their variance.

The next result establishes an upper bound on the expected network regret of \Cref{alg:known} (all missing proofs are in the supplementary material).
\begin{restatable}{theorem}{lintheorem}
\label{th:lin_thm}
Assume each agent runs an instance of \algoname\ with learning rate $\eta > 0$ and i.i.d gossip matrices $W_t$.
% and i.i.d.\ gossip matrices $W_1,W_2,\ldots$.
Then, the expected network regret can be bounded by
    \begin{align}
    \nonumber
        \expRegbar
    &\le
        \frac{N D^2 }{\eta}
        +
        \frac{L^2}{\mu} \eta \bigg({\bar p N + \bar p (1-\bar p) - \sigma_p^2 }
    \\&+ 6 +
    \label{eq:first-bound}
        3 \min\big(\bar p N, \sqrt N\big) \frac{\rho}{1-\rho} \bigg)T\,,
    \end{align}
    where $\rho = \sqrt{\lambdatwo}$.
    In the $p$-uniform case, we have
    \begin{align}
    \nonumber
        &\expRegbar
    \\ &\le
        \frac{D^2}{p\eta}
    +
    \label{eq:second-bound}
        \frac{L^2}{\mu} \eta \lr{8 + 3 \min\big(p N, \sqrt N\big) \frac{\rho}{1-\rho}}T~.
    \end{align}
    If, in addition,
    ${\displaystyle
        \eta =  \frac{(D/L) \sqrt{\mu}}{2\sqrt{ 2p \min(p N, \sqrt N)T}}
    }$,
    then
    \begin{align}
    \nonumber
        \expRegbar
    &\le
      2\sqrt 2  \frac{ DL} {\sqrt{\mu}}\frac{1}{1-\rho}\sqrt{T}
    \\&\times
    \label{eq:third-bound}
        \begin{cases}
                      % N \sqrt p &\text{ if } p\leq 1/N \\
                      %\sqrt N , \text{ for } \leq 1/\sqrt N \\
                      \sqrt N & \text{ if } p\in [1/ N , 1/\sqrt N] \\
                      N^{1/4} / \sqrt p & \text{ otherwise }
        \end{cases}
    \end{align}
\end{restatable}
\begin{comment}
    \begin{restatable}{theorem}{gentheorem}
        \label{th:general_thm}
        With any $\eta>0$, the network regret can be bounded by
%
        \begin{align*}
            \expRegbar& \leq \frac{D^2 }{\eta} + L^2 /\mu T \eta + 3 \eta  N^{} L^2   \frac{\sqrt{\lambdatwo}}{1 - \sqrt{\lambdatwo}} \tilde T\,,
        \end{align*}
%
        where $\lambdatwo$ is the second highest eigenvalue of $\E[W_t^2]$ and $\tilde T = \lr{1- \Pi_{v \in \mathcal V} \lr{1-p_v}}T$ is the expected number of times where there is at least one active agent.
%
        \\
%
        Assuming knowledge of $\E[W_t^2]$, $N$, $D$, $L$, and all the probabilities $\{p_v :~ v \in \mathcal V\}$ by each of the agents and setting
        $\eta = L/D \sqrt{1+ 3 N \frac{\sqrt{\lambdatwo}}{1 - \sqrt{\lambdatwo}}} \sqrt{\tilde T}$,
%
        \[\expRegbar \leq  2 LD \sqrt{1+ 3 N \frac{\sqrt{\lambdatwo}}{1 - \sqrt{\lambdatwo}}} \sqrt{ \tilde T} \,.\]
    \end{restatable}
\end{comment}
%
Bound~\eqref{eq:third-bound} reveals two different regimes based on the value of $p$ in the $p$-uniform case.
Regardless of the regime, the factor inside the brackets in bound~\eqref{eq:third-bound} is less than $\sqrt{N}$ and becomes at most $N^{1/4}$ when $p\to 1$. Note also that, lacking any knowledge on $p$, one can can set $\eta = (D/L)N^{1/4}\sqrt{\mu/T}$ and get the suboptimal bound $\expRegbar \leq 8 D L N^{\frac{3}{4}}\frac{1}{1-\rho}\sqrt T$.

The bounds of \Cref{th:lin_thm} capture the structure of $G$ through the reciprocal of the spectral gap, $\frac{1}{1-\rho}$.
% which is inversely proportional to both $\pmin^2$ and to spectral quantities related to the connectivity of $G$.
In \Cref{sec:lambdatwo}, we give upper bounds on $\frac{1}{1-\rho}$ for an appropriately chosen gossip matrix $W_1$.
Specifically, combining~\eqref{eq:first-bound} with \Cref{th:labgen} and choosing an appropriate $\eta$ which only requires knowing $\pmin$, $N$, and the spectral radius of $G$ we get, under mild conditions on $\{p_v\}_{v \in V}$ (see the appendix),
% $
%     \eta
% =
%     \frac{D \sqrt{\mu} \pmin} {L\sqrt{ T}}
% $,
% which only requires knowing $\pmin$, suffices to obtain the bound
% %
% \begin{equation}
%     \expRegbar
% \le
%     12 DL\frac{\sym(G)}{\pmin}N\sqrt{\frac{T}{\mu}}~.
% \label{eq:rate_diffp}
% \end{equation}
% If additionally
% $\pmin^2\pbar\leq \frac{1}{2\sqrt N}$, setting $
%     \eta
% =
%     \frac{N^{1/4} D \sqrt{\mu} \pmin} {L\sqrt{ T}}
% $ yields
%
\begin{equation}
    \expRegbar
\le
    12 DL\frac{\sym(G)}{\pmin}N^{3/4}\sqrt{\frac{T}{\mu}}~.
\label{eq:rate_diffp_2}
\end{equation}
Concerning the dependence of the network regret on $N$, there is a discrepancy between bound~\eqref{eq:rate_diffp_2} for arbitrary activation probabilities and bound~\eqref{eq:third-bound} for the $p$-uniform case. In the appendix, we show that the factor $N^{3/4}$ in~\eqref{eq:rate_diffp_2} can be brought down to $\sqrt{N}$ provided active agents know $|S_t|$ in each round $t$.

\paragraph{Comparison with previous bounds.}
As mentioned in \Cref{sec:intro}, lower bounds on the standard notion of regret in DOO apply (up to constant factors) to our network regret. Hence, it is fair to compare our bound in \Cref{th:lin_thm} to the bounds previously shown in the DOO literature.
To compare with previous results, we restrict our analysis to the special case when $p_v=1$ for all $v \in V$. In this case, our bound~\eqref{eq:third-bound} is of order of $N^{1/4}\frac{\rho}{1-\rho}\sqrt{T}$. This matches the upper bounds of \citep{hosseini2013online,yan2012distributed}---recall that our global loss is divided by the number of active agents, so the upper bounds for the standard DO setting must be divided by $N$.
If the active graph $G_t$ is an Erdős-Rényi random graph with parameter $q$, our setting reduces to that of \citet{lei2020online} for convex losses and full feedback. As shown in \Cref{sec:E-R}, our analysis recovers the upper bound of order $\frac{N^{1/4}}{q}\frac{\rho}{1-\rho}\sqrt{T}$ in \citep[Theorem 1]{lei2020online} when $\eta$ is tuned based on $N$. In \Cref{erdosrenyi}, we also prove a general bound that holds for all $p$ and $q$ and where $\rho$ is expressed in terms of simple graph-theoretic quantities. When $p=1$ and $q=1$, \citet{pmlr-v247-wan24a} recently show that using accelerated gossip one can achieve a bound of order $\sqrt{\frac{\rho}{1-\rho}T\ln N}$ when $\eta$ is tuned based on both $N$ and $\rho$. Under the same tuning assumptions, our bound~\eqref{eq:third-bound} is instead of order $N^{1/4}\sqrt{\frac{\rho}{1-\rho}T}$.

\paragraph{Lower bound on activation probabilities.}
Our analysis assumes $\sum_v p_v \ge 1$ ensuring that the fraction of rounds with zero active agents is vanishingly small with high probability. If this assumption is dropped, the time horizon $T$ in our bounds is replaced by the expected number $\tilde T = \big(1- \Pi_{v \in \mathcal V} \lr{1-p_v}\big)T$ of time steps when there is at least one active agent (if no agents are active in a give step, then that step does not contribute to the regret). 
However, optimizing the learning rate with respect to $\tilde{T}$ is problematic because this quantity depends on the activation probabilities. On the other hand, note that
$
    \tilde{T}
\le
    \big(1 - (1-\pmax)^N\big)T
\le
    \pmax NT
$.
% $\tilde T$ by $T p_{\min}^{1/2}$ in $\eta$ leads to an additional factor of $p_{\min}^{-1/4}$ (at worst) in the bound.
Hence, when $\pmax$ is known and smaller than $\frac{1}{N}$ (which, in turn, implies that  $\sum_v p_v < 1$), we can tune the learning rate using $\pmax NT < T$.

% \paragraph{When $|S_t|$ is known.} If the number of active agents $|S_t|$ is known, we can run a variant of \algoname\ where $\nabla \ell_t(v)$ is replaced by $\frac N {|S_t|} \nabla \ell_t(v)$ in~\eqref{eq:gr}.
% % 
% % This creates a closer analogy with \cite{lei2020online}, as rescaling the gradients ensure that all matrices $G_t$ have a norm $\|\cdot\|_{\|\|, 1}$ homogeneous to the sum of $N$ gradient vectors.
% % 
% Using this variant yields the bound
% \begin{equation}
% \label{eq:known-S}
%     \expRegbar
% \le
%     \frac{D^2 }{\eta} +  \eta\frac{L^2}{\mu}\left(1 + \frac{3N}{1 - \rho}\right)T
% \end{equation}
% regardless of the activation probabilities.
% % 
% After tuning, the order of this bound is $\sqrt N$ times smaller than that of~\eqref{eq:bound_pv}.
% %
% Hence, not knowing the number of active agents roughly causes an extra $\sqrt N $ factor in the regret bound.

\paragraph{Non stationary activation distributions.}
If the activation probabilities change over time within known bounds $p^{\min}$ and $p^{\max}$, we can still recover the main bound as long as the activations events are independent. In particular, by tuning $\eta$ as a function of $T,\pmin,\pmax$, the expected network regret of \algoname\ is bounded by~\eqref{eq:rate_diffp_2}.

We complement the result of \Cref{th:lin_thm} with a high probability bound on the network regret.
\begin{restatable}{theorem}{gentheoremhp}
\label{th:general_thm_hp}
Assume each agent runs an instance of \algoname\ with learning rate $\eta > 0$.
Then, with probability $ 1- \delta$, the network regret is bounded by
    \begin{align*}
        \Regbar_T &\leq  N \left( \frac{D^2}{\eta} + \frac{L^2}{\mu} \eta T \right)
        \\&+ 3\eta T N \frac{L^2}{\mu}\left (\frac{3}{1- \rho^2}\log\frac{NT^2}{\delta} + 3\right)~.
    \end{align*}
\end{restatable}
There are two notable differences between the bound in expectation provided by \Cref{th:lin_thm} and this one.
First, the high-probability bound has a $1/(1-\rho^2)$ factor instead of $\rho/(1-\rho)$, where the former is smaller than the latter when $\rho < \big(\sqrt{5}-1)/2$.
%
 %, but the second tends to $0$ for small values of $\rho$ (which can only happen when $p$ is large), while $1/{1-\rho^2}$ tends to $1$.
%
This difference in the dependence on $\rho$ is caused by Markov’s inequality, which is used here to bound the deviation probabilities between $\prod W_s$ and $\bone \bone^T/ N$ in the gossiping analysis.
Second, the dependence on $N$ is worse by a factor of $\sqrt{N}$
(however, when $|S_t|$ is known, %and the alternative algorithm with bound~\eqref{eq:known-S} is used, 
this extra $\sqrt{N}$ factor disappears).

%% file: 1b-complexity.tex
\section{\uppercase{Lower bound}}
\label{sec:lower}
Although our notion of network regret is weaker than the one considered in DOO settings with a constant number of agents, the next result shows that when all agents are active with probability $1$, then instances that are hard for any given agent are also hard for most agents. \looseness -1
\begin{restatable}{theorem}{thlower}
\label{th:lb}
 Pick any $\sX$ with diameter $D$ with respect to the Euclidean norm. Pick any $N\ge 4$ multiple of $4$, and let $p_v=1$ for all $v \in V = [N]$. For any DOO algorithm and for any horizon $T$, there exists a connected graph $G=(V,E)$ and $L$-Lipschitz, convex local losses  $\ell_t(v,\cdot)$ for $v \in V$ and $t \in [T]$ such that
\[
    \Regbar_T
\ge
    \frac {3DL}{16} \left(\frac{8\rho}{1-\rho}\right)^{\frac{\alpha}{4}}N^{\frac{1-\alpha}{2}} \sqrt{T}\,,
\]
for all $0 \le \alpha \le 1$, where $\rho = \sqrt{\lambda_2(\E[W^2])}$ and $W = I_N - \frac{\lap(G)}{\lambda_1(G)}$.
\end{restatable}
% For comparison, recall that, with the best possible tuning, our bound~\eqref{eq:third-bound} for $p=1$ is of order $\big(\frac{\rho}{1-\rho}\big)^{1/2}N^{1/4}\sqrt{T}$. 
For $\alpha = 1$, the lower bound of \Cref{th:lb} becomes of order $\big(\frac{\rho}{1-\rho}\big)^{1/4} \sqrt{T}$, showing $\Regbar_T = \tilde{\Omega}(R_T)/N$ because of \citep[Theorem~3]{pmlr-v247-wan24a}.

%% file: 3-spectral-prop.tex
\section{\uppercase{The gossip matrix}}
% Spectral properties of $\E[W_1^2]$ for an appropriate choice of gossip matrix
\label{sec:lambdatwo}
%
%\subsection{}\label{sec:choice}
%
Following the literature on gossip algorithms, we set
\begin{equation}
    \label{eq:def_w}
    W_t = I_N - b\,\lap(\scG_t)\,,
\end{equation}
where $\scG_t= (V,E_t)$ and $b > 0$ is a parameter set so that $b \leq 1/ \lambda_1(\scG_t)$.
One can easily verify that this a gossip matrix for $S_t$. Indeed, it is a symmetric and doubly-stochastic matrix, whose off-diagonal elements $W_t(i,j)$ are zero whenever either $i$ or $j$ are not in $S_t$. In particular, $W_t$ is nonnegative because $\lambda_1(\scG_t) \leq \lambda_1(\scG)$, since $\scG_t$ is obtained by removing edges from $\scG$.
Note also that by knowing $b$ and its active neighborhood, an agent $v$ can compute $W_t(v,\cdot)$, as required by \algoname.
Finally, since the sets $S_1,S_2,\ldots$ of active agents are drawn i.i.d., the matrices $W_1,W_2,\ldots$ are also i.i.d.

\paragraph{A general bound on $\rho$.}
%
%Expliciting the bound in \Cref{th:lin_thm} only requires finding a closed-form for $\rho = \lambdatwo$.
The following result provides a general upper bound that, when applied to the bounds of \Cref{th:lin_thm}, characterizes the dependence of the regret both on the probabilities $p_v$ and on the graph structure (through the Fiedler value $\lambda_{N-1}(G)$ or the condition number $\sym(G)$).
%In the case of non-uniform activation probabilities, we replace the close form of \Cref{th:lam} by a somewhat looser bound on $\lambdatwo$.
\begin{theorem}
    \label{th:labgen}
    If $W_1$ is set according to~\eqref{eq:def_w}, then
    \begin{equation}
    \label{eq:up_bound_Fiedler}
        \rho^2 \leq 1 - b p_{\min}^2  \lambda_{N-1}(\scG)~.
    \end{equation}
    Moreover, for $b = 1/\lambda_1(\scG)$ we have
    \begin{equation}
        \label{eq:up_bound_sym}
        \rho^2 \leq  1 - \frac{\pmin^2}{\sym(\scG)}~.
    \end{equation} 
\end{theorem}
\begin{proof}
Recall $\rho^2 = \lambdatwo$. We have
    \begin{align}
        \lambdatwo &\leq \lambda_2(\E[W_1]) \label{eq:lam_gen1}\\
        &\leq \lambda_2\Big(I - b\,\E\big[\lap(\scG_1)\big]\Big) \nonumber \\ %\label{eq:lam_gen2}\\
        &\leq \lambda_2\big(I-b\,P\lap(\scG) P\big) \nonumber \\ %\label{eq:lam_gen3}\\
        &\leq \lambda_2\big(I - b\,\pmin^2 \lap(\scG)\big) \label{eq:lam_gen4}\\
        &\leq 1-b\,p_{\min}^2 \lambda_{N-1}(\scG)~, \label{eq:lam_gen5}
    \end{align}
    where $P$ is the diagonal matrix such that $P(v,v) = p_v$.
    Now, \eqref{eq:lam_gen1} holds because $W_1$ is symmetric and $W_1^2 \preceq W_1$.
    % \Cref{eq:lam_gen2} and \Cref{eq:lam_gen3} stem from simple rewritings,
    Moreover, $P \lap(\scG) P$ is also symmetric and, clearly, $P \lap(\scG) P \succeq \pmin^2 \lap(\scG)$, implying~\eqref{eq:lam_gen4}. Finally, \eqref{eq:lam_gen5} holds because $\lambda_{N-1}(G)$ is the smallest non-zero eigenvalue of $\lap(G)$.
\end{proof}
%
% When $p \to 0$, the leading factor in the denominator of the upper bound on the regret is $p$, which means that the regret is roughly inversely proportional to $p$ for $p\ll 1$.
This choice of $b$, which is the best possible under the constraint that the gossip matrix is nonnegative, reveals that the regret is naturally controlled by the condition number of $\lap(G)$.

%The second result is an upper bound that uses the symmetry coefficient of the graph.
%
%We know that $d_{\textnormal{min}} +1 \leq \lambda_{1}(\lap(\scG))\leq 2 d_{\textnormal{max}}$, so that in approximately regular graphs, this value is of the order of the degree, which gives a symmetry coefficient of approximately the connectivity (also close to the degree) over the degree.
%
%Plugging these results into \Cref{th:lin_thm} yields two different bounds on the regret, 

\paragraph{The $p$-uniform case.}
We now derive a closed-form expression for $\rho$ in the $p$-uniform case.
\begin{restatable}{theorem}{thlam}
    \label{th:lam}
    If $W_t$ is set according to~\eqref{eq:def_w}, $b = 1/\lambda_1(G)$, then in the $p$-uniform case we have
    \begin{align*}
        \rho^2 = 1 - \frac{2p^2}{\sym(G)}\left(1 - \frac{1-p}{\lambda_1(G)} - \frac{p}{2\sym(G)}\right)~.
    \end{align*}
\end{restatable}
From \Cref{th:lin_thm}, we know that the expected network regret scales with $\frac{\rho}{1-\rho}$. Using \Cref{th:lam} and some simple calculations (see appendix), we get
\[
    \frac{\rho}{1-\rho}  \le \frac{2\sym(G)}{p^2}.
\]
Combining this with \eqref{eq:second-bound}, we immediately get the following result.
%$p \le 
% \min\big(1,\frac{\lambda_1(G)-1}{\lambda_{N-1}(G)}\big)$ 
\begin{cor}
\label{cor:scaling}
Assume each agent runs an instance of \algoname\ with learning rate $\eta > 0$.
If the gossip matrix $W_t$ is chosen as in~\eqref{eq:def_w} with $b = 1/\lambda_1(G)$ and $\eta$ is tuned with respect to $p$ and $N$, the expected network regret can be bounded by
\begin{equation}
\label{eq:e-r-tuned-p}
    \expRegbar = \scO\left(\frac{\sym(G)}{p}\min\left(\sqrt{N},\frac{N^{1/4}}{\sqrt p}\right)\sqrt{T}\right)
\end{equation}
for all $p\le 1$.
\end{cor}
This bound captures the intuition that bottlenecks in $G$ (causing a small Fiedler value or a high condition number) negatively impact the regret due to a slower propagation of the information in the network.

To better visualize the dependence on the graph structure, we study specific graphs of particular practical importance.
Specifically, we give results in the $p$-uniform case for cliques, strongly regular graphs, and grids (see \Cref{fig:fig_graph}).
% \begin{figure}[t]
%     \centering
%     \subfigure[$36$-node clique]{%
%         \includegraphics[width=0.3\columnwidth, valign=c]{Figures/figs_graphs_clique-crop.pdf}
%         \label{fig:sub1}}
%     \hfill
%     \subfigure{%
%         \includegraphics[width=0.55\columnwidth, valign=c]{Figures/lambda2_clique-crop.pdf}
%         \label{fig:sub2}}

%     \vskip\baselineskip

%     \subfigure[$6\times6$ lattice]{%
%         \includegraphics[width=0.3\columnwidth,valign =c]{Figures/figs_graphs_lattice-crop.pdf}
%         \label{fig:sub5}}
%     \hfill
%     \subfigure{%
%         \includegraphics[width=0.55\columnwidth,valign=c]{Figures/lambda2_lattice-crop.pdf}
%         \label{fig:sub6}}

%     \vskip\baselineskip

%     \subfigure[$6\times6$ grid]{%
%         \includegraphics[width=0.3\columnwidth, valign=c]{Figures/figs_graphs_grid-crop.pdf}
%         \label{fig:sub3}}
%     \hfill
%     \subfigure{%
%         \includegraphics[width=0.55\columnwidth,valign=c]{Figures/lambda2_grid-crop.pdf}
%         \label{fig:sub4}}

%     \caption{Empirical estimate $\rhohat$ compared to $\rho$ and the upper bound $\rhoup$~\eqref{eq:up_bound_sym} for $b = 1/\lambda_1(\scG)$ plotted as a function of $p \in [0,1]$.}
%     \label{fig:fig_graph}
% \end{figure}

\begin{figure}[t]
    \centering
    % First row of subfigures
    \subfloat[$36$-node clique]{%
        \includegraphics[width=0.3\columnwidth, valign=c]{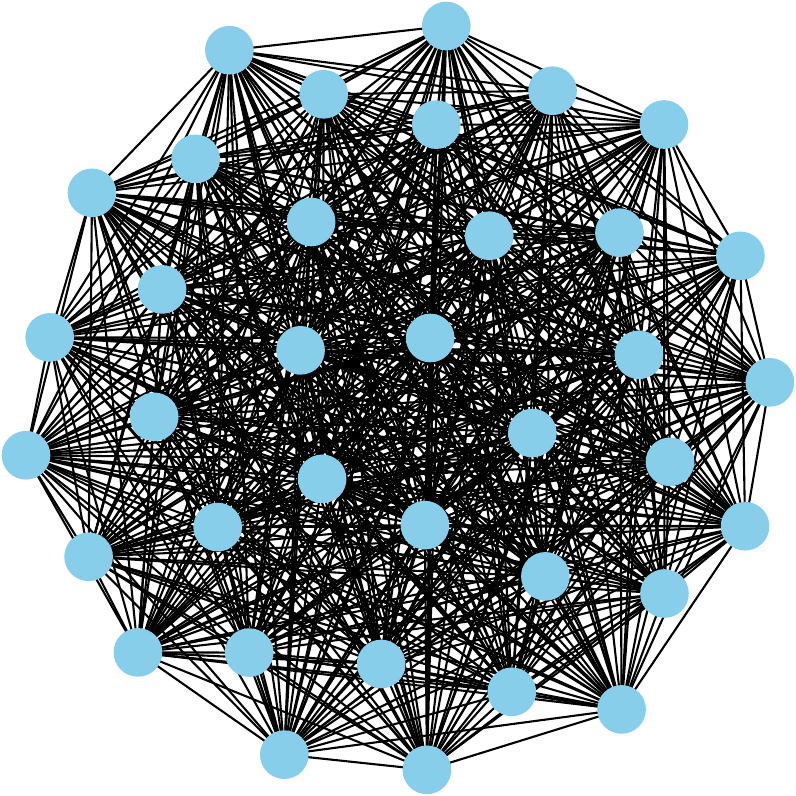}
        \label{fig:sub1}}
    \hfill
    \subfloat{%
        \includegraphics[width=0.55\columnwidth, valign=c]{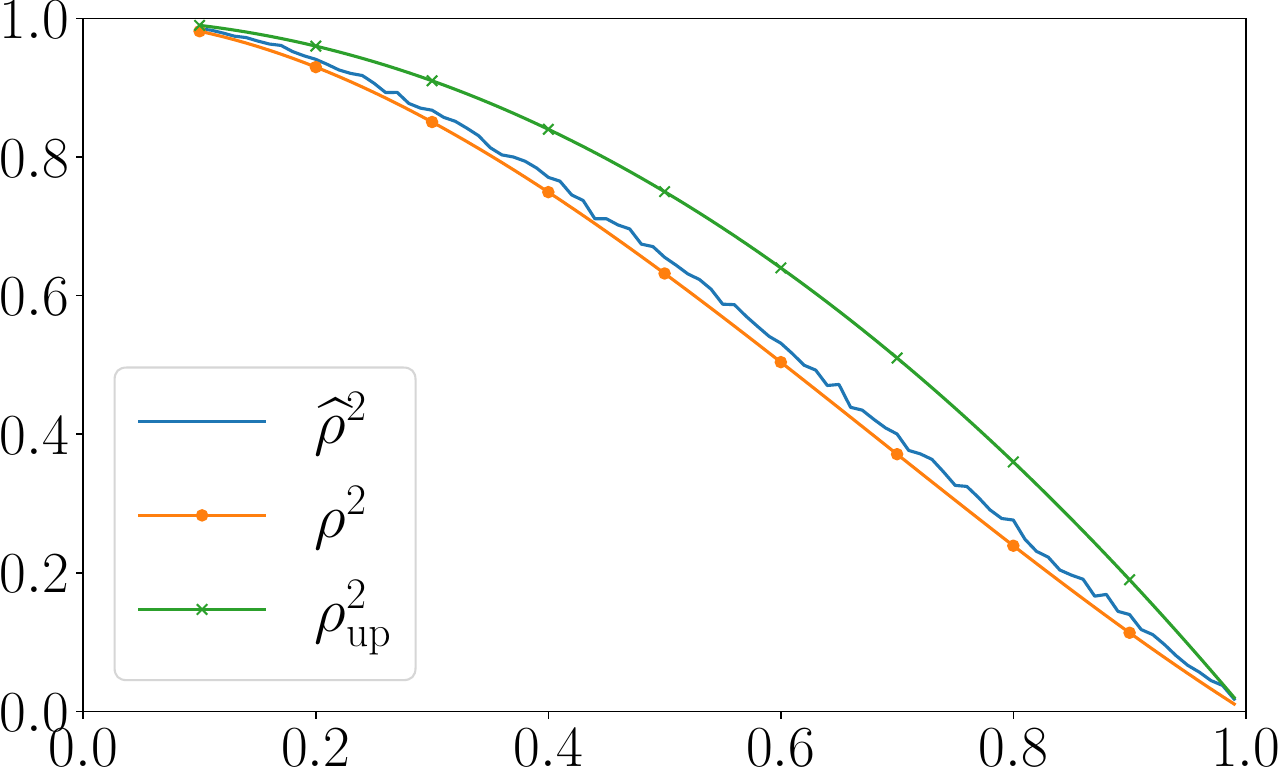}
        \label{fig:sub2}}

    \vskip\baselineskip

    % Reset subfigure counter so that numbering continues as (b)
    \setcounter{subfigure}{1} 
    \subfloat[$6\times6$ lattice]{%
        \includegraphics[width=0.3\columnwidth,valign=c]{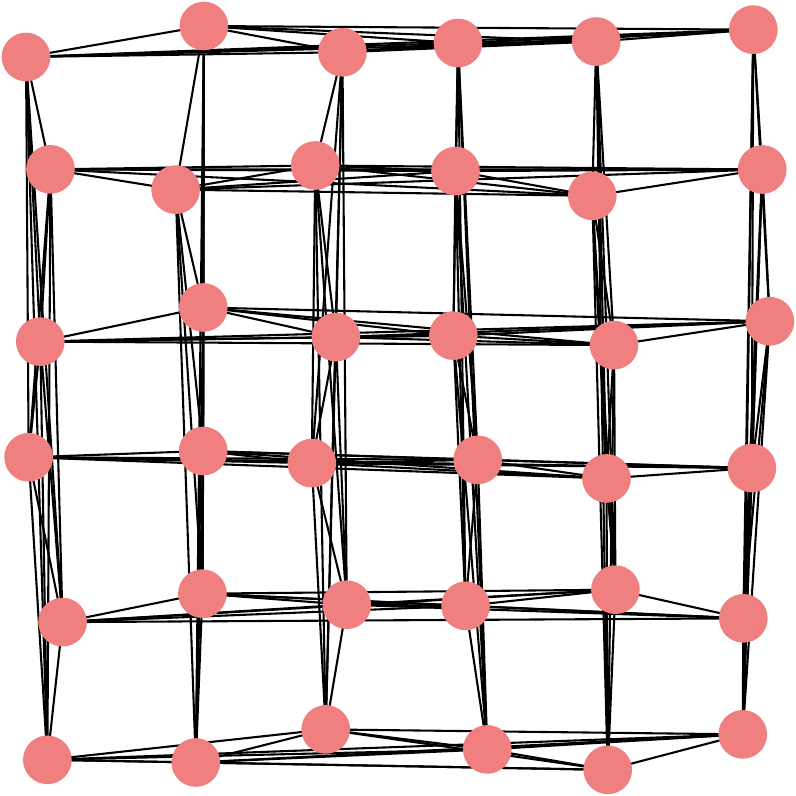}
        \label{fig:sub5}}
    \hfill
    \subfloat{%
        \includegraphics[width=0.55\columnwidth,valign=c]{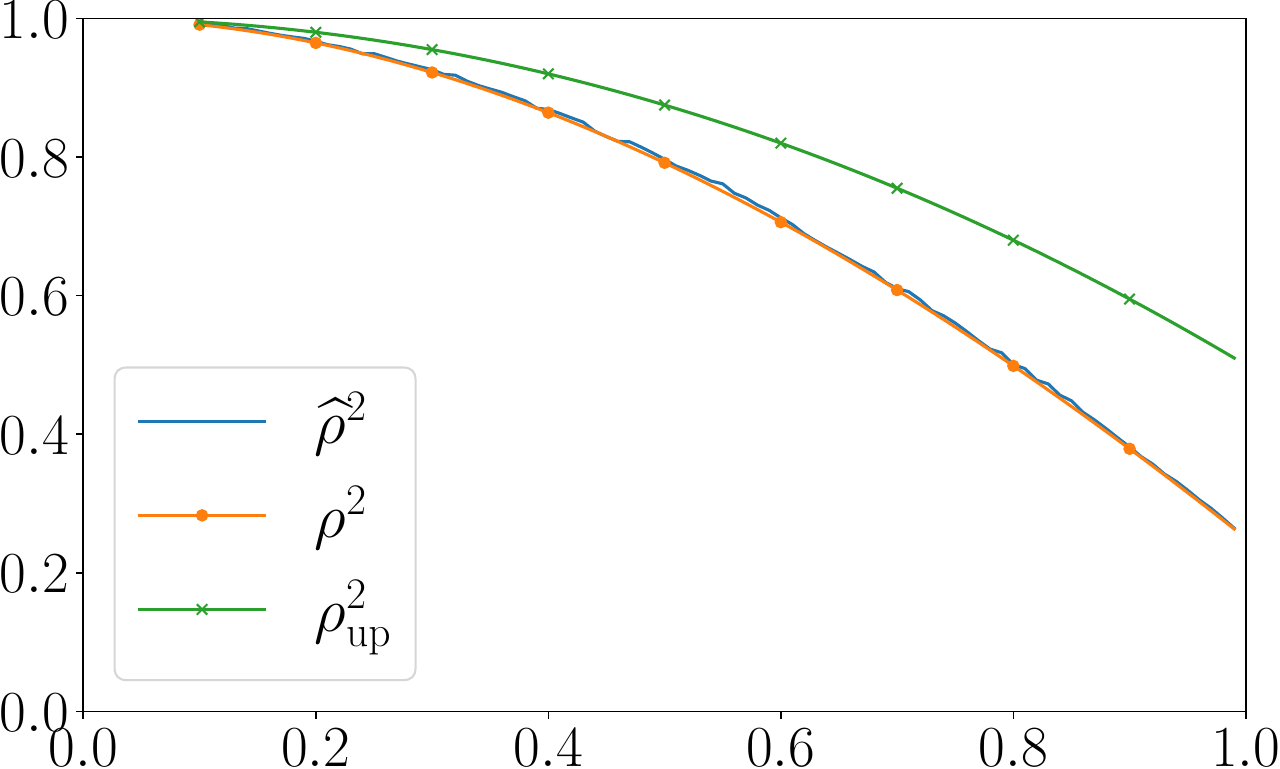}
        \label{fig:sub6}}

    \vskip\baselineskip

    % Reset subfigure counter so that numbering continues as (c)
    \setcounter{subfigure}{2} 
    \subfloat[$6\times6$ grid]{%
        \includegraphics[width=0.3\columnwidth, valign=c]{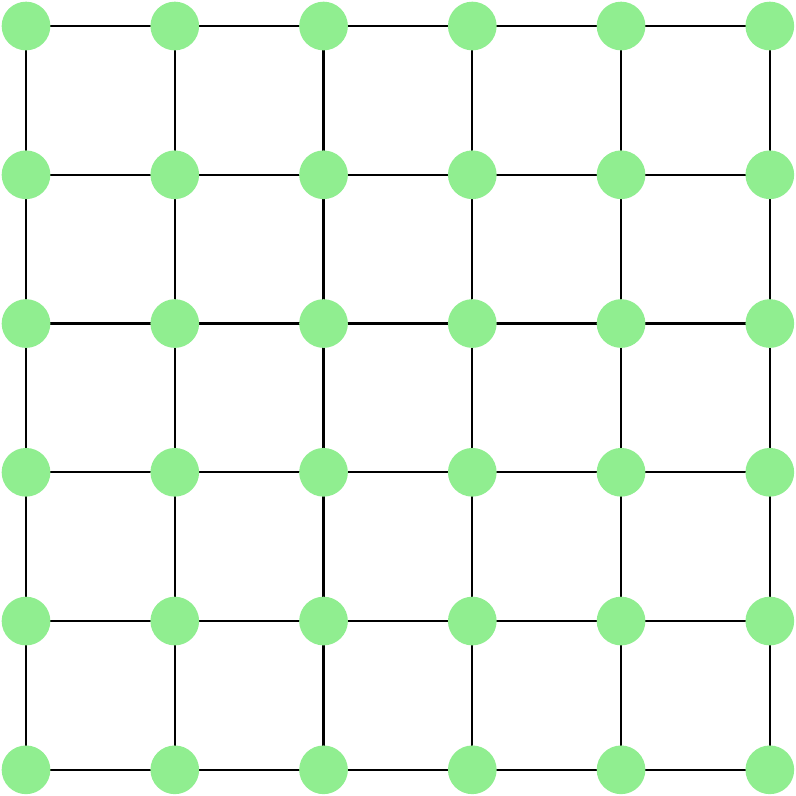}
        \label{fig:sub3}}
    \hfill
    \subfloat{%
        \includegraphics[width=0.55\columnwidth,valign=c]{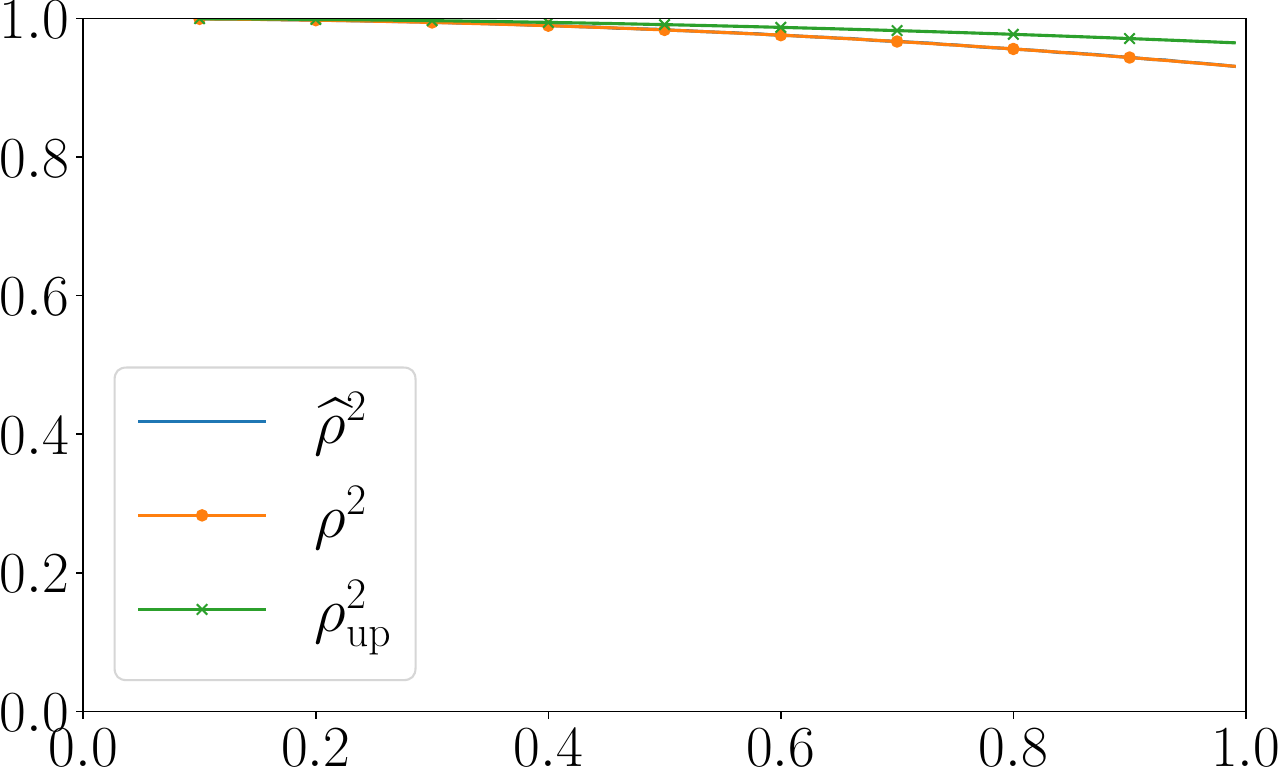}
        \label{fig:sub4}}

    \caption{Empirical estimate $\hat{\rho}$ compared to $\rho$ and the upper bound $\hat{\rho}_{\text{up}}$~\eqref{eq:up_bound_sym} for $b = 1/\lambda_1(\scG)$ plotted as a function of $p \in [0,1]$.}
    \label{fig:fig_graph}
\end{figure}

\textbf{Clique.} We have $\lambda_1(\scG) = \lambda_{N-1}(\scG) = N$ and 
\[
    \rho^2
=
    1 - 2p^2 + p^2\left(\frac{2(1-p)}{N} + p\right)~.
% \le
%    1 - p^2~.
\]
\textbf{Strongly regular.} Let $\scG$ be strongly regular with parameters $k$ (the degree of any node), $m$ (the number of common neighbors for any two adjacent nodes), and $n$ (the number of common neighbors for any two nonadjacent nodes). Then $\lambda_1(\scG) = k-s$, $\lambda_{N-1}(\scG) = k-r$, and
\begin{equation*}
        \rho^2 \le 1-\frac{k-r}{k-s}p^2~,
\end{equation*}
where
$\begin{cases}
        r = \frac{m-n + \sqrt{(m-n)^2+ 4 (k-n)}}{2}~,
    \\
        s = \frac{m-n - \sqrt{(m-n)^2+ 4 (k-n)}}{2}~.
\end{cases}$

In particular, when $\scG$ is the lattice graph, i.e. the graph with vertices $[M]^2$ and an edge between any  two vertices in the same rows or columns (yielding  $k=2M-2$, $m=M-2$ and $n=2$), % with $N= M^2$ vertices
$\rho \le 1 - \frac 1 2 p^2$.

\textbf{$2$-dim grid.} We have $\lambda_1(\scG) = 4-4 \cos(\pi (M-1)/M)$ and $\lambda_{N-1}(G) = 2-2\cos(\pi/M)$, where $M = \sqrt{N}$ is the grid side length. Then
\begin{align*}
    \rho^2
%&=
%    f\left(2-2\cos\frac{\pi}{M}, 4-4\cos\frac{\pi (M-1)}{M}, p\right)
%\\&
\le
    1 - \frac{1-\cos(\pi/M)}{2-2\cos(\pi (M-1)/M)} p^2~.
\end{align*}
Note that $\frac{1-\cos(\pi/M)}{2-2\cos(\pi (M-1)/M)} \sim \frac{\pi^2}{4M^2}$, which goes to zero when $M\to\infty$.

\Cref{fig:fig_graph} shows the empirical behavior of $\rho^2$ for $b = 1/\lambda_1(\scG)$. The quantity $\rhohat$ is the second eigenvalue of $W_1^2$ averaged over $1000$ different draws of active agents, where each agent is activated with probability $p$ ranging from $0$ to $1$. We also plot the exact value $\rho^2$ (\Cref{th:lam}) and its upper bound $\rhoup$~\eqref{eq:up_bound_sym}.

\Cref{fig:fig_graph} reveals that for dense graphs (e.g., clique and lattice), $\rho^2$ decreases quickly as $p\to 1$, implying a better regret rate. For the clique we have $\rho = 0$, implying an expected regret rate of order $\sqrt{T}$, which is independent of $N$---see~\eqref{eq:second-bound}. On the other hand, in sparse graphs $\rho$ may remain high. For example, in the grid $\rho > 0.9$ for all $p$. Note also that $\rhoup$ approximates $\rho^2$ well, especially when $p$ is small.

%% file: 4-erdosrenyi.tex
\section{\uppercase{random edges}}
\label{sec:E-R}

We now study a setting where, after agents are activated, edges between pairs of active agents are independently deleted with probability $1-q$.
More specifically, given a graph $G = (V,E)$, the active graph $G_t = (S_t,E_t)$ at time $t$ is defined by $\Pr\big((i,j)\in E_t \big) = q\,\Pr(i,j \in S_t) \Ind{(i,j) \in E}$. When $\Pr(i,j \in S_t) = 1$ for all distinct $i,j \in V$ we recover the model of \citet{lei2020online}. In the $p$-uniform case, we write $G_t \sim \myG(G,p,q)$. Note that $G_1,G_2,\ldots$ is i.i.d.\ because $S_1,S_2,\ldots$ is i.i.d.; moreover, if $W_t$ is chosen as in~\eqref{eq:def_w}, then $W_1,W_2,\ldots$ is also an i.i.d.\ sequence. Using~\eqref{eq:second-bound}, we can prove the following result.
%
% In the $1$-uniform case (all activation probabilities equal to $1$), we recover the setting studied by \citep{lei2020online}.
%
\begin{restatable}{cor}{erdosrenyi}
\label{erdosrenyi}
Assume each agent runs an instance of \algoname\ with learning rate $\eta > 0$.
If the gossip matrix $W_t$ is chosen as in~\eqref{eq:def_w} with $b = 1/\lambda_1(G)$, then
\begin{align*}
  \rho^2 = 1 - \frac{2p^2q}{\sym(G)}\left(1 - \frac{1-pq}{\lambda_1(G)} - \frac{pq}{2\sym(G)}\right)~. %1 - 2bqp^2\left(1 - b + bqp\left(1 - \frac{\lambda_{N-1}(G)}{2}\right)\right) \lambda_{N-1}(G)~.
\end{align*}
By tuning $\eta$ with respect to $p$ and $N$, the expected network regret on $G_1,G_2,\ldots$ drawn i.i.d.\ from $\myG(G,p,q)$ can be bounded by
\begin{equation}
\label{eq:e-r-tuned}
    \expRegbar = \scO\left(\frac{\sym(G)}{pq}\min\left(\sqrt{N},\frac{N^{1/4}}{\sqrt p}\right)\sqrt{T}\right)~.
\end{equation}
\end{restatable}
%
% If we set $p=1$ in~\eqref{eq:e-r-bound} we recover the bound of \citet[Theorem~1]{lei2020online} as a special case (however, here $\ellbar_t$ is an average of local losses and our regret is not per agent, but averaged over the agents).

%% file: 5-exp.tex
\section{\uppercase{Experiments}}
We empirically evaluate \algoname\ on synthetic data and compare it with \dogd\ \citep{lei2020online}.
While \algoname\ can deal with arbitrary values of $p$ and $q$, \dogd\ is designed for settings with $p=1$ (agents are always active) and $0 < q \le 1$ (edges of $G$ are active with probability $q$). To run \dogd\ when $p < 1$ we feed a zero gradient vector to instances run by agents that are inactive on that round.

Our synthetic data are generated based on the distributed linear regression setting of \citet{yuan2020distributed}. In particular, the agents' decision space $\sX$ is the $10$-dimensional Euclidean ball of radius $2$ centered in the origin. The local loss functions are $\ell_t(v,\bx) = \frac{1}{2}\big(\langle \bw_t(v),\bx\rangle - y_t(v)\big)^2$ for all $v\in V$ and $\bx\in\sX$. The feature vectors $\bw_t(v)$ are generated independently, by picking each coordinate independently and uniformly at random in $[-1,1]$. The labels $y_t(v)$ are generated according to $y_t(v) = \varepsilon_t(v)$ for $1 \le v < \lceil N/2 \rceil$ and $y_t(v) = \langle \bw_t(v),\bone\rangle + \varepsilon_t(v)$ for the remaining agents, where $\varepsilon_t(v)$ is independent Gaussian noise (zero mean and unit variance). Hence, the local losses of half of the agents are random noise.

Each of the following experiments is run with $|V| = N = 36$ and $T = 1000$. Plots are averages over $20$ repetitions, where repetitions use the same labels and feature vectors and only agent (and possibly edge) activations are drawn afresh in each repetition. Both algorithms are tuned according to the theoretical specifications (ignoring constant factors): we set $\eta = \big(p \min(p N, \sqrt N)T\big)^{-1/2}$ for \algoname\ (see \Cref{th:lin_thm}) and $\eta = N^{-1/4}T^{-1/2}$ for \dogd.

\begin{figure}
    \centering
    \subfloat[Clique $p=0.5, q=1$.]{%
    \centering
        \includegraphics[width=0.23\textwidth]{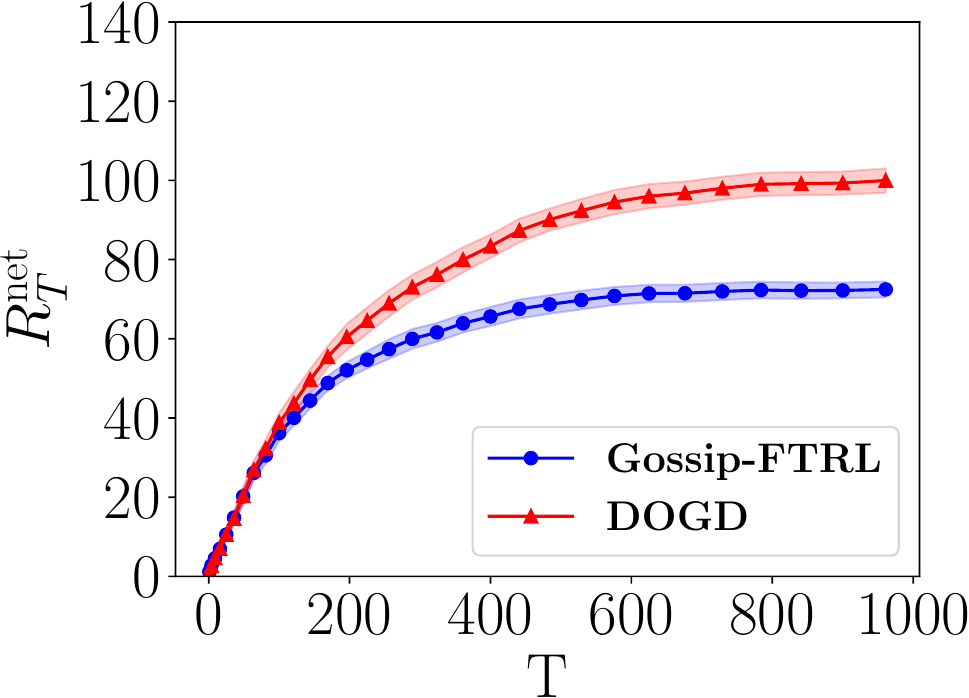}
        \label{fig:DFTRL.5-1}}
         \hfill
        \subfloat[Grid $p=0.5, q=1$.]{%
        \centering
        \includegraphics[width=0.23\textwidth]{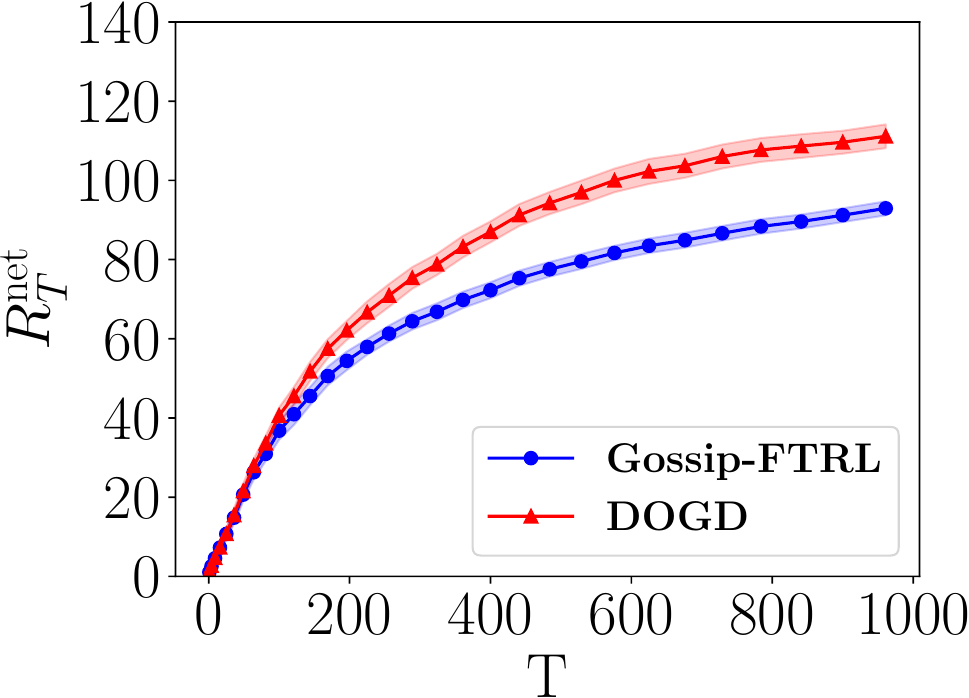}
        \label{fig:DFTRL.5-2}}

        \subfloat[Clique $p=0.5,q=0.05$.]{%
    \centering
        \includegraphics[width=0.23\textwidth]{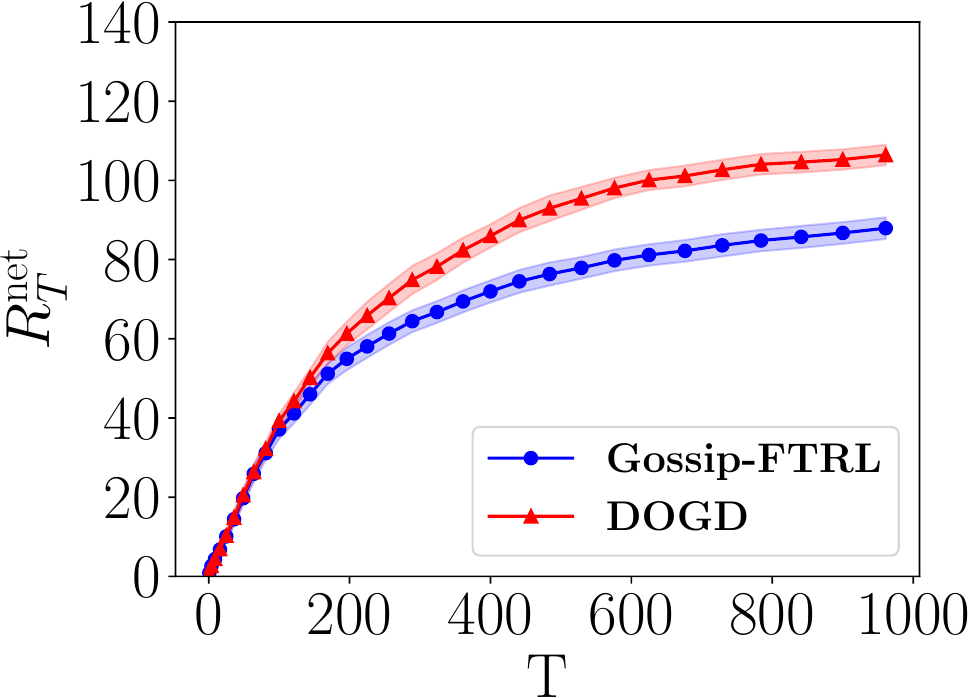}
        \label{fig:DFTRL.5-3}}
         \hfill
        \subfloat[Grid $p=0.5,q = 0.5$.]{%
        \centering
        \includegraphics[width=0.23\textwidth]{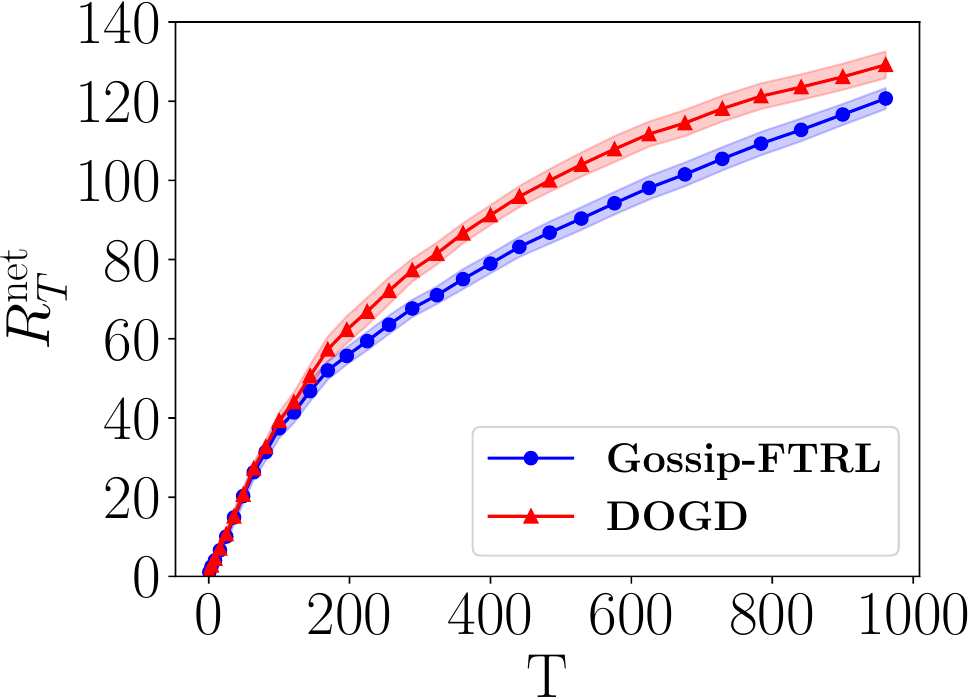}
        \label{fig:DFTRL.5-4}}
    \caption{Growth over $T=1000$ steps of the network regret of \algoname\ and \dogd on a clique and on a grid for $N=36$ and for different choices of $p,q$.}
\end{figure}

Our experiments show that \algoname\ performs consistently better than \dogd, although the difference is not huge. Both algorithms are surprisingly robust to sparsity induced by low values of $q$ when $G$ is dense (\Cref{fig:DFTRL.5-1} and \Cref{fig:DFTRL.5-3}). When $G$ is sparse though, the regret goes up much quickly as $q$ becomes smaller (\Cref{fig:DFTRL.5-2} and \Cref{fig:DFTRL.5-4}). Figure~\ref{fig:DFTRL_3D} shows the behavior of \algoname\ and \dogd\ on a grid for pairs $(p,q)$ in the set $\{0.4,0.6,0.8\}^2$. Figure~\ref{fig:DFTRL_TP-2} shows that, for \algoname, $\Regbar_T$ scales approximately with $\frac{1}{p^2}$ as predicted by \eqref{eq:third-bound}---at least for sufficiently small values of $p$---and \dogd\ exhibits a similar behavior. Finally, Figure~\ref{fig:DFTRL_two_cliques} shows the impact of $\lambda_{N-1}(G)$ on the network regret of \algoname. The regret decreases as $\lambda_{N-1}(G)$ is increased by adding more edges to the bottleneck between the two cliques.
% The elevated regret on the left side of the plot is due to the presence of a bottleneck. It gradually decreases as the bottleneck dissipates. 

\begin{figure}
\centering         
\vspace*{-5mm}\includegraphics[width=0.4\textwidth]{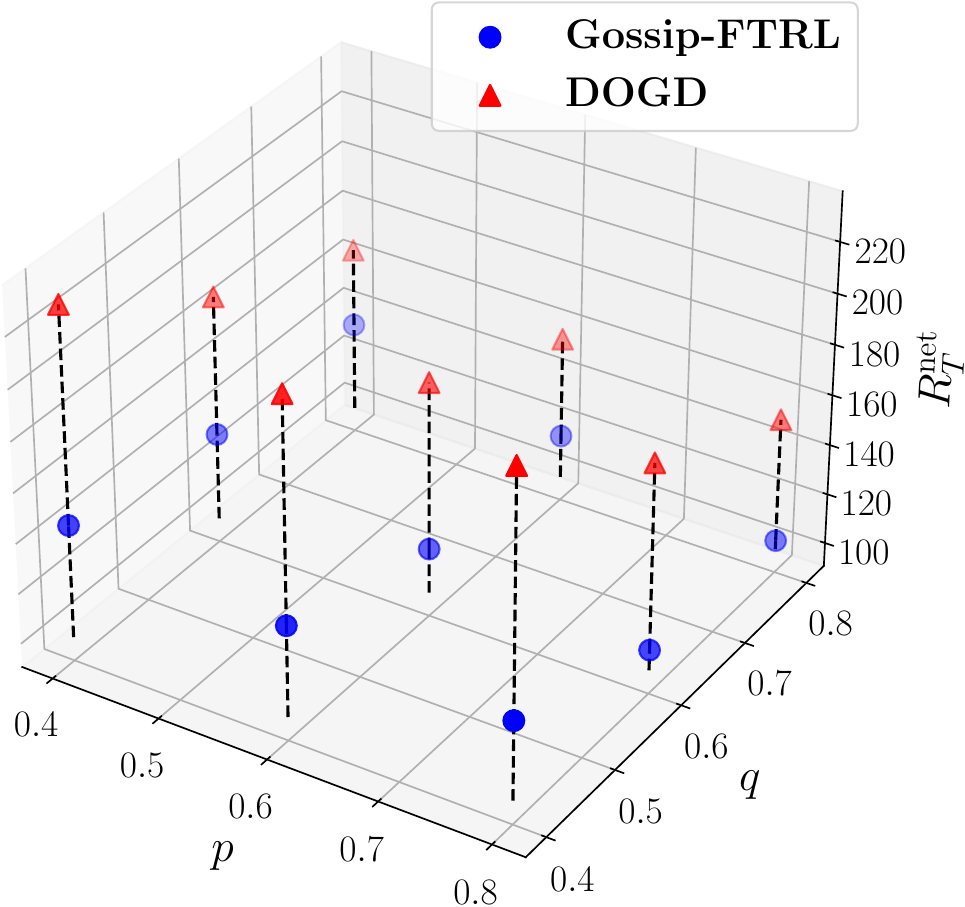}
\caption{
Network regret of \algoname\ and \dogd after $T=1000$ steps on a grid with $N=36$. 
}\label{fig:DFTRL_3D}
\end{figure}

\begin{figure}
    \centering
    \vspace*{-5mm}\includegraphics[width=0.4\textwidth]{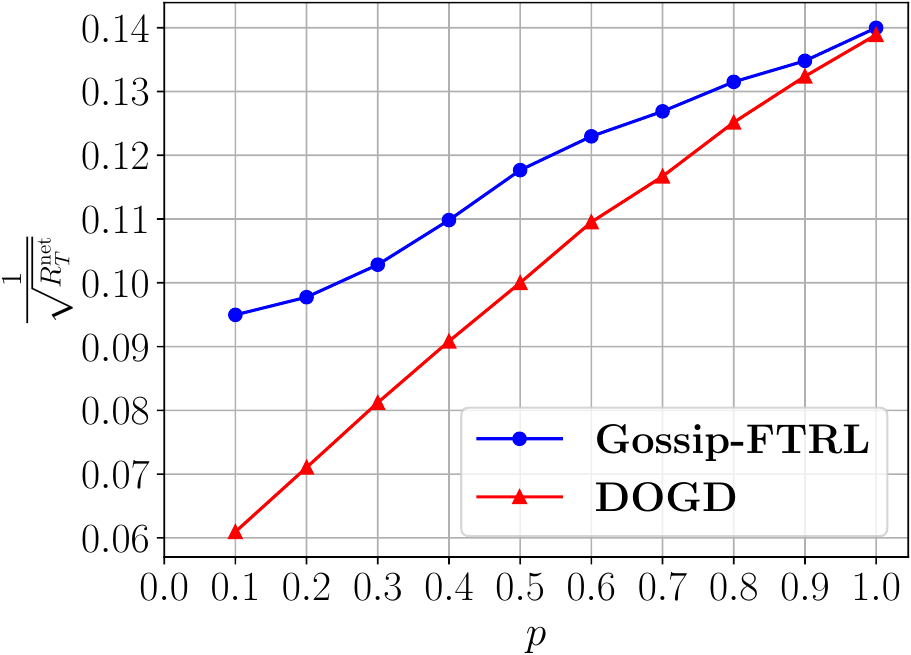}
    \caption{Plot of $\big(\Regbar_T\big)^{-1/2}$ for \algoname\ on a clique for $p \in [0,1]$ and $T=1000$.}
    \label{fig:DFTRL_TP-2}
\end{figure}

\begin{figure}
    \centering
    \vspace*{-5mm}\includegraphics[width=0.4\textwidth]{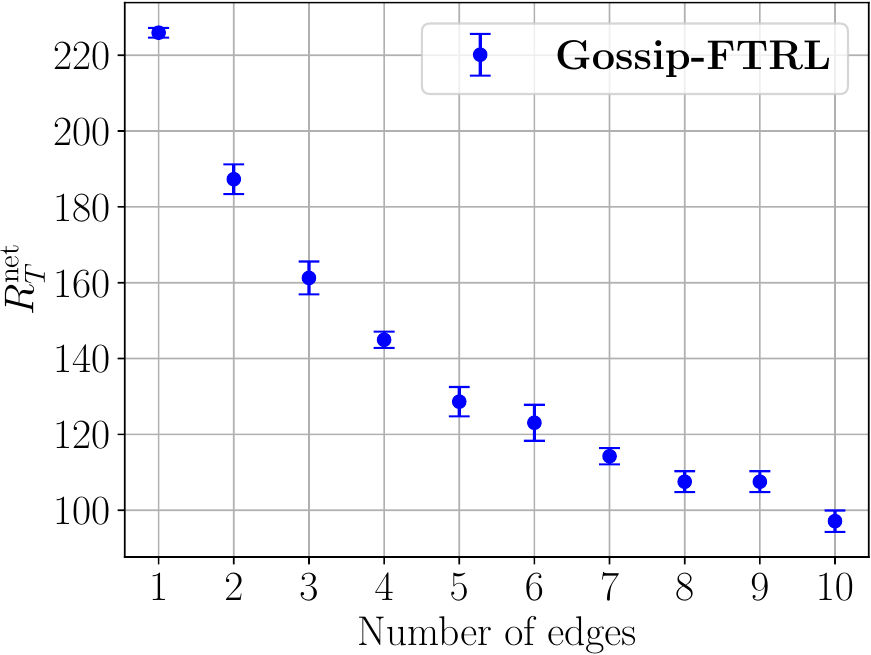}
    \caption{Network regret of \algoname\ after $T=1000$ steps when $p=0.5$, $q=1$, and $G$ is made up by two cliques joined by a varying number of random edges.}
    \label{fig:DFTRL_two_cliques}
\end{figure}

%% file: 6-limitations.tex
\section{\uppercase{open problems}}
There is a gap between our upper bounds and the lower bound of \Cref{th:lb}. When $p=1$, a distributed algorithm using accelerated gossiping techniques matches the lower bound \citep{pmlr-v247-wan24a}. It is unclear if the same techniques could be applied in our more general setting, where $1/N < p \le 1$, to improve on our results. Finally, our bounds require tuning based on preliminary knowledge of $p$ in the $p$-uniform case, or $\pmin,\pbar$ in the general case. It is unclear whether we could get similar results when this information is unavailable.

%% file: 10_appendix.tex
\section{Additional related works}

% \paragraph{Distributed Optimization and Gossip Algorithms}
Distributed optimization (DO), now central to federated learning, dates back to the  work of \cite{bertsekas1991some}, originally applied to parallel computation. 
Much of the research in DO focuses on gossip algorithms, introduced by \cite{boyd2005gossip, boyd2006randomized} to address the distributed averaging problem, especially in scenarios where communication is expensive.
These algorithms involve randomly selecting a neighbor for information exchange, or \textit{gossiping}.
The concept later expanded to include weighted averaging of information from all neighbors using weights collected in a gossip matrix.
\cite{nedic2010constrained} extended gossip methods to distributed optimization, combining projected gradient descent with gossip-based averaging of iterates.
Typically, the convergence rate of gossip algorithms is inversely related to the spectral gap of the gossip matrix.

% \paragraph{Distributed Optimization over Randomized Graphs}
A key constraint frequently considered in distributed optimization is that the algorithm should be robust to random network topologies.
This can arise not only from unstable communication channels \citep{nedic2014distributed}, but randomization can also be leveraged to reduce communication costs while preserving performance \citep{lei2020online}.
Other constraints, considered in the literature but less relevant to this work, include event-related communication and time delays \citep{yang2019survey}.
Recent advances in DO also concern accelerated gossip algorithms that allow for accelerated rates with respect to the number of agents \citep{pmlr-v247-wan24a}.

% \paragraph{Online Distributed Optimization}
In the online DOO setting, \cite{yan2012distributed} proposed a (sub)gradient descent algorithm with regret bounds of $\mathcal O (\sqrt T)$ for the convex and $\mathcal O (\log T)$ for the strongly convex case.
\cite{hosseini2013online} later introduced an online dual-averaging algorithm, also achieving $\mathcal O (\sqrt T)$ regret for convex losses.
\cite{yuan2021distributed} extended this to long-term constraints.
Time varying graph networks were considered by \cite{hosseini2016online, mateos2014distributed}, who proved regret rates under the assumption that the union of communication networks over any $m$ time steps is strongly connected.
\cite{lei2020online} studied the special case where communication networks are Erdös-Rényi graphs, in which each edge has a probability  $q$ of existing at each round.
They proposed a gradient descent algorithm and proved regret upper bounds for the convex and strongly convex case, also extending their result to the bandit feedback framework---see also \citep{shi2021federated}.
%
% In the convex case, their regret bound scales with $\sqrt T$ and $\sqrt N \rho  /(1-\rho)$.

% \paragraph{The problem of device availability}
In this work, we address the problem of device unavailability, a topic explored in federated learning from various angles. 
For example, availability patterns, such as diurnal cycles, can violate the assumption of data independence, as active agents may disproportionately represent certain populations (e.g., by geographic location) \citep{eichner2019semi, amiri2021convergence}.
%
%\cite{amiri2021convergence} tackled this in a centralized setting by using scheduling techniques to reduce bias caused by unavailable agents.
%
Device unavailability has not been addressed in the context of DOO, except indirectly in \citep{hosseini2016online}, which focuses on time-varying graphs, thus tackling the case in which isolated devices are unavailable for communication.
Our approach differs in that inactive agents do not have an associated loss function and do not contribute to the global loss. \citet{raginsky2011decentralized} consider a notion of information structure replacing the communication network. However, their results are based on a specific linear structure and a horizon-dependent communication radius within which agents can freely exchange information.
%
%In online learning with multiple agents, some works such as \citep{cesa2020cooperative, cesa2021multitask} consider inactive agents that do not contribute to the regret, but these works assume either homogeneous losses or personalized objectives. 
% While some works like \cite{cesa2020cooperative, cesa2021multitask} address inactive agents in online learning, they assume homogeneous losses or personalized objectives.

\section{Preliminary results}

\begin{lemma}[Regret decomposition]
    \label{lem:decomp}
    The network regret
    \[\Regbar_T = \sum_{t=1}^T \frac{1}{|S_t|} \sum_{v\in S_t} \ellbar_t\big(S_t,x_t(v)\big)
    -
    \min_{x \in \sX} \sum_{t=1}^T \ellbar_t(S_t,x)\]
    can be decomposed in the following way
    \begin{align*}
        \Regbar_T  \leq \underbrace{ 3 \sum_{t=1}^{T} \sum_{u \in V} \frac{\Ind{u\in S_t}}{|S_t|}L \left \| x_t(u) - y_t\right\|}_{(A)}
        +     \underbrace{\sum_{t=1}^{T} \sum_{v \in V} \frac{\Ind{v\in S_t}}{|S_t|} \vp{\nabla \loss_{t}(x_t(v),v)}{y_t - x^{*}}}_{(B)}\, ,
    \end{align*}
    for any $y_t\in \mathcal X$, with the convention that $0/0 = 0$ when $|S_t|=0$.
    %for fractions such as  $\Ind{u\in S_t}/|S_t|$.
\end{lemma}
This will be particularly useful when setting $y_t$ as the prediction of a omniscient agent knowing the gradients of all incurred losses up to time $t-1$. In this case, Term $(B)$ is the part of the regret related to the loss incurred by the prediction of the omniscient agent, and Term $(A)$ is the part of the regret related to the deviations with respect to these predictions.

\begin{proof}
    By definition of the regret,
    \begin{align*}
        \Regbar_T =
        &\sum_t^{T} \sum_{u \in V} \sum_{v \in V} \frac{1}{|S_t|^2}\lr{\loss_t(x_t(u),v) - \loss_t(x^{*}, v)}\Ind{u\in S_t} \Ind{v \in S_t}
        \\&= \sum_{t=1}^{T} \sum_{u \in V} \sum_{v \in V} \frac{1}{|S_t|^2}\lr{\loss_t(x_t(u),v) - \loss_t(y_t,v) +\loss_t(y_t,v) - \loss_t(x^{*}, v)}  \times \Ind{u\in S_t}\Ind{v \in S_t}
        \\&\leq\sum_{t=1}^{T} \sum_{u \in V} \sum_{v \in V}  \frac{1}{|S_t|^2}\lr{L \left \| x_t(u) - y_t\right\|+\loss_t(y_t,v) - \loss_t(x^{*}, v)}
        \times \Ind{u\in S_t} \Ind{v \in S_t}\,,
    \end{align*}
    because $\loss_t(\cdot, v)$ is $L$-Lipschitz over the set $\mathcal{X}$ w.r.t the norm $\|\cdot\|$, i.e. $|\loss_t(x, v)-\loss_t(y, v)| \leq L\|x-y\|, \forall ~ x, y \in X$.
    Next, because we need to introduce individual gradients, we add and remove each $\loss_t({x}_t,v)$:
    \begin{align*}
        \Regbar_T
        &\leq
        \sum_{t=1}^{T} \sum_{u \in V} \sum_{v \in V} \lr{L \left \| x_t(u) - y_t\right\|+\loss_t(y_t,v) - \loss_t(x_t(v), v) + \loss_t(x_t(v), v) - \loss_t(x^{*}, v)}
        \times \frac{1}{|S_t|^2}  \Ind{u\in S_t} \Ind{v \in S_t}
        \\&\leq
        \sum_{t=1}^{T} \sum_{u \in V} \sum_{v \in V}  \lr{L \left \| x_t(u) - y_t\right\|+L \left \| x_t(v) - y_t\right\| + \loss_{t}(x_t(v), v) - \loss_t(x^{*}, v)}
        \times  \frac{1}{|S_t|^2} \Ind{u\in S_t} \Ind{v \in S_t}\,,
    \end{align*}
    where we used again the Lipschitzness of the loss functions.
    Then by convexity of $\loss_t(\cdot, v)$,
    \begin{align*}
        \Regbar_T\leq
        \sum_{t=1}^{T} \sum_{u \in V} \sum_{v \in V}
        \frac{1}{|S_t|^2}\lr{L \left \| x_t(u) - y_t\right\|+L \left \| x_t(v) - y_t\right\| + \vp{\nabla \loss_{t}(x_t(v),v)}{x_t(v) - x^{*}}}
        \times \Ind{u\in S_t} \Ind{v \in S_t}\,,
    \end{align*}
    which can be rewritten as
    \begin{align*}
        \Regbar_T& \leq  \sum_{t=1}^{T} \sum_{u \in V} \frac{\Ind{u\in S_t}}{|S_t|}L \left \| x_t(u) - y_t\right\|
        + \sum_{t=1}^{T} \sum_{v \in V} \frac{\Ind{v\in S_t}}{|S_t|}L \left \| x_t(v) - y_t\right\|
        \\
        & \quad +
        \sum_{t=1}^{T} \sum_{v \in V} \frac{\Ind{v\in S_t}}{|S_t|} \vp{\nabla \loss_{t}(x_t(v),v)}{x_t(v) - x^{*}}
        \\& = 2 \sum_{t=1}^{T} \sum_{u \in V} \frac{\Ind{u\in S_t}}{|S_t|}L \left \| x_t(u) - y_t\right\|
        + \sum_{t=1}^{T} \sum_{v \in V} \frac{\Ind{v\in S_t}}{|S_t|} \vp{\nabla \loss_{t}(x_t(v),v)}{x_t(v) -y_t + y_t - x^{*}}
        \\&\leq 2 \sum_{t=1}^{T} \sum_{u \in V} \frac{\Ind{u\in S_t}}{|S_t|}L \left \| x_t(u) - y_t\right\|
        +\sum_{t=1}^{T} \sum_{v \in V} \frac{\Ind{v\in S_t}}{|S_t|} L \left \| x_t(v) - y_t\right\|
        \\
        & \quad+ \sum_{t=1}^{T} \sum_{v \in V} \frac{\Ind{v\in S_t}}{|S_t|} \vp{\nabla \loss_{t}(x_t(v),v)}{y_t - x^{*}}\, ,
    \end{align*}
    again by the Lipschitzness of the losses.
    Finally,
    \begin{equation}
        \label{eq:regret_decomposition}
        \Regbar_T\leq \underbrace{ 3 \sum_{t=1}^{T} \sum_{u \in V} \frac{\Ind{u\in S_t}}{|S_t|}L \left \| x_t(u) - y_t\right\|}_{(A)}
        +
        \underbrace{\sum_{t=1}^{T} \sum_{v \in V} \frac{\Ind{v\in S_t}}{|S_t|} \vp{\nabla \loss_{t}(x_t(v),v)}{y_t - x^{*}}}_{(B)}
    \end{equation}
    concluding the proof.
\end{proof}

ce\begin{lemma} \label{lem:lambW}   Assuming that for $k=1\ldots T$, $W_k$ are doubly stochastic matrices and i.i.d., we have, $\forall v \in V$, $\forall s,t \in [T]$ such that $ t>s$,
 \begin{align}
        \mathbb{E}\left[\left(W_{t} \cdots  W_{s+1} e_v - \frac \bone N \right)^T \left(W_{t} \cdots  W_{s+1} e_v - \frac \bone N\right)  \right]
        &\leq {e_v^T e_v} \left\| \mathbb{E}[W_1 W_1^{\top}] - \frac 1 N \bone \bone^\top \right\|_{\textnormal{op}}^{t-s} \nonumber\\
        &\leq \lambdatwo^{{t-s}}\,. \label{eq:induc}
    \end{align}   
\end{lemma}
    This can be derived exactly as in the proof of \citep[Lemma 2]{lei2020online}.
    For completeness, we provide a quick justification.
\begin{proof}
    Let $\tilde W_{k} =W_{k}- \frac 1 N \bone \bone^\top$ and assume
    \[
    \E\left[\left\|W_{k-1} \cdots  W_{s+1} e_v - \frac \bone N \right\|_2^2\right]\leq e_v^T e_v \left\| \mathbb{E}[W_1 W_1^{\top}] - \frac 1 N  \bone \bone^\top \right\|_{\textnormal{op}}^{k-s-1} \qquad \text{for some $k-1>s$.}
    \]
    Let $\mathcal{F}_{k-1}$ be the $\sigma$-algebra generated by all random events up to time $k-1$. We have that
    \begin{align*}
        \E\left[\left\|W_{k}^{\top} \cdots  W_{s+1}^{\top} e_v - \frac \bone N \right\|_2^2\right]
        &=\E\left[e_v^T \tilde W_{s+1}^{\top} \cdots \tilde W_{k-1}^{\top} \tilde W_{k}^{\top}\tilde W _k \tilde  W_{k-1} \cdots \tilde W_{s+1} e_v  \right]\\
        &=\E\left[e_v^T \tilde W_{s+1}^{\top} \cdots \tilde W_{k-1}^{\top}  \E[\tilde W_{k}^{\top}\tilde W _k  \mid \mathcal{F}_{k-1}] \tilde  W_{k-1} \cdots  \tilde W_{s+1} e_v  \right] \\
        &=\E\left[e_v^T \tilde W_{s+1}^{\top} \cdots \tilde W_{k-1}^{\top}  \E[\tilde W_{1}^{\top}\tilde W _1] \tilde  W_{k-1} \cdots  \tilde W_{s+1} e_v  \right] \tag{by independence of $W_k$}\\
        & \leq   \left\| \mathbb{E}[W_1 W_1^{\top}] - \frac 1 N \bone \bone^\top \right\|_{\textnormal{op}} e_v^T e_v \left\| \mathbb{E}[W_1 W_1^{\top}] - \frac 1 N \bone \bone^\top \right\|_{\textnormal{op}}^{k-s-1}\\
        & \leq   \lambdatwo e_v^T e_v \left\| \mathbb{E}[W_1 W_1^{\top}] - \frac 1 N \bone \bone^\top \right\|_{\textnormal{op}}^{k-s-1}
    \end{align*}
    which by induction, suffices to prove \Cref{eq:induc}.
\end{proof}

\section{Proof of \Cref{th:lin_thm}}
\lintheorem*

\begin{proof}
The proof relies on the use of an omniscient agent knowing the gradients of all incurred losses up to time $t-1$.

    Let us define the quantities $\barz_t$ and $\bar{g}_t$
    \begin{align*}
        &\bar{g}_t = \frac{1}{N}\sum_{v \in V} g_t(x_t(v), v)
        = \sum_{v \in V} \frac{\mathds{1}(v \in S_t)}{N} \nabla \loss_{t}(x_t(v),v)
        \\
        &\barz_t = \frac{1}{N} \sum_{v \in V} z_t(v).
    \end{align*}

    Then the decision of the omniscient agent is defined as
    \begin{equation*}
        \bar{x}_t = \operatorname{argmin}_{x \in X}\left\{\langle \bar{z}_t, x\rangle+\frac{1}{\eta} \psi(x)\right\}.
    \end{equation*}
    Note that
    \begin{equation}
        \barz_{t+1} =  \barz_{t} + \bar{g}_t.
    \end{equation}
    The proof of the theorem relies on \Cref{lem:decomp}, where $y_t$ is set to $\bar x_t$.
    \begin{equation}
        \label{eq:regret_decomposition_1}
        \Regbar_T\leq \underbrace{ 3 \sum_{t=1}^{T} \sum_{u \in V} \frac{\Ind{u\in S_t}}{|S_t|}L \left \| x_t(u) - {\bar x}_t\right\|}_{(A)}
        +
        \underbrace{\sum_{t=1}^{T} \sum_{v \in V} \frac{\Ind{v\in S_t}}{|S_t|} \vp{\nabla \loss_{t}(x_t(v),v)}{{\bar x}_t - x^{*}}}_{(B)}\, .
    \end{equation}

    We start by analyzing the general case.
    Let us focus on Term $(B)$ first.
    \begin{align*}
        \E \left[\sum_{t=1}^{T} \sum_{v \in V} \frac{\Ind{v\in S_t}}{|S_t|} \vp{\nabla \loss_{t}(x_t(v),v)}{\bar x_t - x^{*}} \right]
        &= \sum_{t=1}^{T} \sum_{v \in V}  \E \lr{\frac{\Ind{v\in S_t}}{1 +\sum_{u \in \mathcal V \setminus v} \Ind{u \in S_t}}}\E\left[\vp{\nabla \loss_{t}(x_t(v),v)}{\bar x_t - x^{*}}\right]\\
        &= \sum_{t=1}^{T} \sum_{v \in V}  p_v c_v \E\left[\vp{\nabla \loss_{t}(x_t(v),v)}{\bar x_t - x^{*}}\right]\\
        &\leq \max_{v\in \mathcal V} c_v \sum_{t=1}^{T} \sum_{v \in V} p_v \E\left[\vp{\nabla \loss_{t}(x_t(v),v)}{\bar x_t - x^{*}}
        \right]
        \,,
    \end{align*}
    where $c_v = \E \lr{\frac{1}{1 +\sum_{u \in \mathcal V \setminus v} \Ind{u \in S_t}}}\leq 1$. Recall that $\tilde T = \big(1- \Pi_{v \in \mathcal V} \lr{1-p_v}\big)T$ denotes the expected number  of time steps when there is at least one active agent.

    In the $p$-uniform case for example, $c_v=\frac{\tilde{T}}{T p N}\leq 1$. This holds because, on the one hand,
    \[\sum_{v \in V} \E\left[\frac{\Ind{v\in S_t}}{|S_t|}\right] = \sum_{v \in V}  p c_v = N p c_v\]
    due to all $c_v$ being equal. On the other hand,
    \[\sum_{v \in V} \E \left[ \frac{\Ind{v\in S_t}}{|S_t|}\right] = \Pr \lr{S_t \neq \emptyset }= \frac {\tilde T} T \,.\]
    Since $\bar x_t$  are the predictions of FTRL on linear losses  $\vp{\bar g_t}{ \cdot}$, we know from standard FTRL analysis \citep[Corollary~7.9]{orabona2019modern},
    \begin{equation}  \label{eq:FTRLresult}\frac{1}{N} \sum_{t=1}^{T} \sum_{v \in V}\vp{\nabla \loss_{t}(x_t(v),v)}{\bar x_t - x^{*}} \Ind{v\in S_t} \leq \frac{\psi(x^*)}{\eta} + \frac{L^2}{\mu}  \sum_{t=1}^T\eta \frac{|S_t|^2}{N^2}\end{equation}
    which by taking expectation and using the independence of $S_t$ and $x_t(v)$ leads to
    \[  \E\left[\sum_{t=1}^{T} \sum_{v \in V} p_v\vp{\nabla \loss_{t}(x_t(v),v)}{\bar x_t - x^{*}} \right] \leq \frac{\psi(x^*) N}{\eta   } + \frac{L^2}{\mu}   \eta \sum_{t=1}^T\Big(N \bar p + \bar p(1-\bar p) - \sigma_p^2 \Big)\,, \]
    where $\bar p = \frac 1 N \sum_{v\in \mathcal{V}} p_v$ and $\sigma_p^2 = \frac 1 N \sum_{v\in \mathcal{V}} p_v^2 - \bar p^2$.

    This holds because $\E[|S_t|^2] = \E[|S_t|]^2 + \textnormal{Var}(|S_t|)$, and $|S_t|$ is the sum of independent Bernoulli of parameter $p_v$, so that $\textnormal{Var}(|S_t|) = \sum_{v\in \mathcal{V}} p_v(1-p_v)$, which can also be written as $N\bar p(1- \bar p) - N \sigma_p^2$.
    Hence
    \begin{align*}
        \E \left[\sum_{t=1}^{T} \sum_{v \in V} \frac{\Ind{v\in S_t}}{|S_t|} \vp{\nabla \loss_{t}(x_t(v),v)}{\bar x_t - x^{*}} \right]
        &\leq  \max{c_v}\lr{  \frac{\psi(x^*) N }{\eta    } + \frac{L^2}{\mu}   \sum_{t=1}^T\eta  (N \bar p + \bar p (1-\bar p) - \sigma_p^2 )} \\
        &\leq  \lr{  \frac{\psi(x^*) N }{\eta    } + \frac{L^2}{\mu}   \sum_{t=1}^T\eta  (N \bar p + \bar p (1-\bar p) - \sigma_p^2 )} \,.
    \end{align*}
    Regarding Term $(A)$, since
    \begin{equation*}
        \bar{x}_t = \operatorname{argmin}_{x \in X}\left\{\langle \bar{z}_t, x\rangle+\frac{1}{\eta} \psi(x)\right\}
    \end{equation*}
    and
    \begin{equation*}
        x_t(v) = \operatorname{argmin}_{x \in X}\left\{\langle z_t(v), x\rangle+\frac{1}{\eta} \psi(x)\right\}\, ,
    \end{equation*}
    we have
    \begin{equation}
        \label{eq:z}
        \left\|x_t(v)-\bar{x}_t\right\| \leq \eta/\mu  \left\|z_t(v)-\bar{z}_t\right\|_{*}\,,
    \end{equation}
thanks to the duality between strong convexity and smoothness \cite[Theorem 6.11]{orabona2019modern}.
    For any $t \in [T]$ and any $v \in [N]$, we have
    \begin{align*}
        \bz_{t+1}  &= W_t \bz_{t} + \bg_t
        = W_t W_{t-1} \bz_{t-1} + W_t \bg_{t-1} + \bg_t
        = \sum_{s=1}^{t-1} W_t \cdots W_{s+1} \bg_s + \bg_t\,.
    \end{align*}

    Simultaneously, we have
    \begin{align*}
        \bar{z}_{t+1} = \frac 1 N \sum_{s=1}^t \bone^\top \bg_s,
    \end{align*}
    so that
    \begin{align*}
        \bz_{t+1} - \bone \bar{z}_{t+1} &=  \sum_{s=1}^t W_t \cdots W_{s+1} \bg_s + \bg_t- \frac 1 N \bone \bone^\top \bg_s\\
        & = \sum_{s=1}^{t-1} \left[ W_t \cdots  W_{s+1} - \frac 1 N \bone \bone^\top \right] \bg_s +\bg_t - \frac 1 N \bone \bone^\top \bg_t\,.
    \end{align*}

    In turn,
    \begin{align*}
        z_{t+1}(v) -  \bar{z}_{t+1} &= (\bz_{t+1} - \bone \bar{z}_{t+1})^T e_v\\
        & =  \sum_{s=1}^{t-1} \bg_s^T \left[ W_t \cdots  W_{s+1} - \frac 1 N \bone \bone^\top \right] ^T e_v + \bg_t^\top (I - \frac 1 N \bone \bone^\top) e_v,\\
    \end{align*}
    so that we can compute :
    \begin{align*}
        \left \|z_{t+1}(v) - \barz_{t+1} \right \|_* &= \left \|\sum_{s=0}^{t-1} \left(  \sum_{u=1}^N \left ([W_t \cdots  W_{s+1}]_{u,v} -  \frac{1}{N} \right) g_{s}(x_{s}(u), u) \right) + g_{t}(x_t(v), v) - \bar{g}_{t}\right\|_*
        \\&
        \leq \left \|\sum_{s=0}^{t-1} \left(  \sum_{u=1}^N \left ([W_t \cdots  W_{s+1}]_{u,v} -  \frac{1}{N} \right) g_{s}(x_{s}(u), u)\right)\right\|_* + \left\|  g_{t}(x_t(v), v) - \bar{g}_{t}\right\|_*
        \\&
        \\&
        \leq \left \|\sum_{s=0}^{t-1} \left(  \sum_{u=1}^N \left ([W_t \cdots  W_{s+1}]_{u,v} -  \frac{1}{N} \right) g_{s}(x_{s}(u), u)\right)\right\|_* + \left\| g_{t}(x_t(v), v) - \bar{g}_{t}\right\|_*
        \\& \leq \sum_{s=0}^{t-1} \sum_{u=1}^N  \left|[W_t \cdots  W_{s+1}]_{u,v} -  \frac{1}{N} \right| \left\| g_{s}(x_{s}(u), u)\right\|_* + \left\| g_{t}(x_t(v), v) - \bar{g}_{t}\right\|_*
        \\& \leq \sum_{s=0}^{t-1} \sum_{u=1}^N  \max_{u \in V}\left|[W_t \cdots  W_{s+1}]_{u,v} -  \frac{1}{N} \right| \left\| g_{s}(x_{s}(u), u)\right\|_* + \left\| g_{t}(x_t(v), v) - \bar{g}_{t}\right\|_*
    \end{align*}
    
    which yields
    \begin{align}
    \|z_{t+1}(v) -  \bar{z}_{t+1}\|_*
        & \leq  \sum_{s=0}^{t-1}  |S_s| L \left\|  W_t \cdots  W_{s+1} e_v - \frac 1 N \bone  \right\|_{\infty} +  2 L \,
        \end{align}
        Because $\left\| g_{s}(x_{s}(u), u)\right\|_* \leq \Ind{u \in S_t} L$ and $\sum_{u \in V}^N \Ind{u \in S_t} = | S_t|$.
        By taking the expectation on each side, 
        \begin{align*}
        \E[ \|z_{t+1}(v) -  \bar{z}_{t+1}\|_*]& \leq  \sum_{s=0}^{t-1} \bar p N L \E\left[ \left\|  W_{t} \cdots  W_{s+1} e_v - \frac 1 N \bone  \right\|_{\infty}\right] + 2   L\\
        & \leq  \sum_{s=0}^{t-1} \bar p N L \E\left[ \left\|  W_{t} \cdots  W_{s+1} e_v - \frac 1 N \bone  \right\|_{2} \right]+ 2   L \,.
    \end{align*}

    At the same time, we also have
    \begin{align*}
         \E[ \|z_{t+1}(v) -  \bar{z}_{t+1}\|_*]
        \leq  \sum_{s=1}^{t-1} \sqrt N L \E\left[\left\|  W_{t} \cdots  W_{s+1} e_v - \frac 1 N \bone  \right\|_{2}\right] + 2 L \,,
    \end{align*}
    because \begin{align} \left \|z_{t+1}(v) - \barz_{t+1} \right \|_* &\le \left \|\sum_{s=0}^{t-1}  \sum_{u=1}^N \left ([W_t \cdots  W_{s+1}]_{u,v} -  \frac{1}{N} \right) g_{s}(x_{s}(u), u)\right\|_* + \left\| g_{t}(x_t(v), v) - \bar{g}_{t}\right\|_* \nonumber \\
      &\leq \sum_{s=0}^{t-1} \sum_{u=1}^N \left |[W_t \cdots  W_{s+1}]_{u,v} -  \frac{1}{N} \right| \max_{u \in V} \left \|g_{s}(x_{s}(u), u)\right\|_* + 2L\nonumber\\
      &=  \sum_{s=0}^{t-1} \left \|W_t \cdots  W_{s+1} e_v -  \frac{1}{N} \bone \right\|_1 \max_{u \in V} \left \|g_{s}(x_{s}(u), u)\right\|_* + 2L \nonumber \\
       & \leq \sum_{s=0}^{t-1} \sqrt N \left \|W_t \cdots  W_{s+1}  e_v -  \frac{1}{N} \bone \right\|_2 \max_{u \in V} \left \|g_{s}(x_{s}(u), u)\right\|_* + 2L \label{eq:z_usual}\end{align}
    Hence, combing the above two inequalities on $\|z_{t+1}(v) -  \bar{z}_{t+1}\|_*$ with \cref{eq:z},
    \begin{equation}
        \label{eq:ptA}
        \E\left[\sum_{t=1}^T \sum_{u \in \mathcal{V}} \frac{1}{|\mathcal{S}_t|} \mathds{1}(u \in \mathcal{S}_t) L \|x_t(u) - \bar{x}_t\|\right]
        \leq \eta    \frac{L^2}{\mu}  \left( 2+\min( \bar p N, \sqrt N)  \frac{\rho}{1 - \rho} \right)\tilde T  \,.
    \end{equation}
    In turn, 
    \[\expRegbar \leq    \frac{D^2 N }{\eta    }+ \frac{L^2}{\mu}  \eta \sum_{t=1}^T  \lr{ 6 + (N \bar p + \bar p (1-\bar p) +  \sigma_p^2 )T/{\tilde T} + 3 \min( \bar p N, \sqrt N) \frac{\rho}{1 - \rho}} \tilde T\,. \]
    Consequently
    \[\expRegbar \leq    \frac{D^2 N }{\eta    }+ \frac{L^2}{\mu}  \eta \sum_{t=1}^T  \lr{ 6 + \frac{\bar p N + \bar p (1-\bar p) - \sigma_p^2 } {p_{\min}} + 3 \min( \bar p N, \sqrt N)  \frac{\rho}{1 - \rho}} \tilde T\,. \]

    We also have 
        \[\expRegbar \leq    \frac{D^2N }{\eta    }+ \frac{L^2}{\mu}  \eta \sum_{t=1}^T  \lr{ 6 + (N \bar p + \bar p (1-\bar p) +  \sigma_p^2 ) + 3 \min( \bar p N, \sqrt N) \frac{\rho}{1 - \rho}}T\,, \]
    which is less tight in general but sufficient with the assumption that $\sum_{v \in V} p_v \ge 1$.
    This directly yields \Cref{{eq:first-bound}}

    We now turn to the $p$-uniform case.
    The refinement in this case is due to the fact that we can compute $c_v$.
    This makes it possible to refine our bound of Term $(B)$.
    Specifically, we are interested in the expectation of Term $(B)$,
    \begin{align*}
        \E \left[\sum_{t=1}^{T} \sum_{v \in V} \frac{\Ind{v\in S_t}}{|S_t|} \vp{\nabla \loss_{t}(x_t(v),v)}{\bar x_t - x^{*}} \right]
        &= \sum_{t=1}^{T} \sum_{v \in V} \E \left[  \frac{\Ind{v\in S_t}}{|S_t|} \right]\E\left[\vp{\nabla \loss_{t}(x_t(v),v)}{\bar x_t - x^{*}}\right]\\
        &= \sum_{t=1}^{T} \sum_{v \in V} \frac{1- (1-p)^N}{N}\E\left[\vp{\nabla \loss_{t}(x_t(v),v)}{\bar x_t - x^{*}}
        \right]\\
        &=  \frac{1-(1-p)^N}{N} \E\left[\sum_{t=1}^{T} \sum_{v \in V}\vp{\nabla \loss_{t}(x_t(v),v)}{\bar x_t - x^{*}}
        \right]\\
        &\leq  \frac{\tilde T}{T N} \E\left[\sum_{t=1}^{T} \sum_{v \in V}\vp{\nabla \loss_{t}(x_t(v),v)}{\bar x_t - x^{*}}
        \right]\,,
    \end{align*}
    where the second inequality comes from the independence of $S_t$ and $\bar{x}_t$ (whose update only depends on the history) and $x_t(v)$.

    Since $\bar x_t$  are the predictions of FTRL on linear losses  $\vp{\bar g_t}{ \cdot}$, we know
    \[  \frac{1}{N} \sum_{t=1}^{T} \sum_{v \in V}\vp{\nabla \loss_{t}(x_t(v),v)}{\bar x_t - x^{*}} \Ind{v\in S_t} \leq \frac{\psi(x^*)}{\eta} + \frac{L^2}{\mu}  \sum_{t=1}^T\eta |S_t|^2/N^2\]
    which by similar arguments yields
    \[ p \E\left[\sum_{t=1}^{T} \sum_{v \in V}\vp{\nabla \loss_{t}(x_t(v),v)}{\bar x_t - x^{*}} \right] \leq \frac{\psi(x^*) N}{\eta  } + \frac{L^2}{\mu}  \sum_{t=1}^T\eta N p + p(1-p) \,.\]

    Hence
    \begin{align}
        \label{eq:partB_unifp}
        \E \left[\sum_{t=1}^{T} \sum_{v \in V} \frac{\Ind{v\in S_t}}{|S_t|} \vp{\nabla \loss_{t}(x_t(v),v)}{\bar x_t - x^{*}} \right]
        &\leq  \frac{\tilde T}{T}\lr{  \frac{\psi(x^*) }{\eta    p} + \frac{L^2}{\mu}   \sum_{t=1}^T\eta  (1 + (1-p)/N)} \,,
    \end{align}
    which concludes our bound of the expectation of Term $(B)$.

    The analysis of Term $(A)$ is unchanged for the case of uniform $p_v$, and yields:
    \begin{equation}
        \label{eq:ptA-2}
        \E\left[\sum_{t=1}^T \sum_{u \in \mathcal{V}} \frac{1}{|\mathcal{S}_t|} \mathds{1}(u \in \mathcal{S}_t) L \|x_t(u) - \bar{x}_t\|\right]
        \leq \eta    \frac{L^2}{\mu} \min(Np,\sqrt N)  \frac{\rho}{1 - \rho} \tilde T\,.
    \end{equation}

    Combining with \Cref{eq:partB_unifp}, we have
    \begin{align*}
        \expRegbar
        &\leq  \frac{\tilde T}{T}\lr{  \frac{\psi(x^*) }{\eta    p}} + \frac{L^2}{\mu}     \eta \lr{ 6 + 1+(1-p)/N + 3 \min(Np,\sqrt N)  \frac{\rho}{1 - \rho}} \tilde T \\
        &
        \leq  \lr{1-(1-p)^N}\lr{  \frac{\psi(x^*) }{\eta    p}} + \frac{L^2}{\mu}     \eta \lr{8+ 3 \min(Np,\sqrt N) \frac{\rho}{1 - \rho}} \tilde T
        \\
        &
        \leq    \frac{D^2\min(1, Np) }{\eta    p}+ \frac{L^2}{\mu}     \eta \lr{8 + 3 \min(Np,\sqrt N)  \frac{\rho}{1 - \rho}} \tilde T\,.
    \end{align*}
    With the assumption that $\sum_{v \in V} p_v \ge 1$, this easily yields \Cref{eq:second-bound}:
    \begin{align*}
        \expRegbar
        \leq    \frac{D^2}{\eta    p}+ \frac{L^2}{\mu}     \eta \lr{8 + 3 \min(Np,\sqrt N)  \frac{\rho}{1 - \rho}}  T\,.
    \end{align*}

    Finally, \Cref{eq:third-bound} follows from simple computations.
\end{proof}

\section{Proof of \Cref{th:general_thm_hp}}

\subsection{Preliminary result}

The main difference between the proof of the bound in expectation and that of the high probability bound is the use of the following Lemma to bound the deviation between $ W_{t} \cdots  W_{s+1}$ and $ \frac 1 N \bone \bone^{\top}$.
\begin{lemma}
    \label{cons}
     Assuming that for $k=1\ldots T$, $W_k$ are doubly stochastic matrices and i.i.d., we have, $\forall v \in V$, $\forall s,t \in [T]$ such that $ t>s$,
    \begin{equation*}
        \Pr\left(\left \|  W_{t} \cdots  W_{s+1} e_v - \frac 1 N \bone  \right\|_{2} \geq \epsilon \right) \leq \frac{\lambda_2(\E[W^2])^{t-s}}{\epsilon^2}.
    \end{equation*}
    When $t-s  \geq \frac{3\log{\epsilon^{-1}}}{\log{\lambda_2[W^2]^{-1}}} = t^{*}$, we have
    \begin{equation*}
        \Pr\left(\left \|  W_{t} \cdots  W_{s+1} e_v - \frac 1 N \bone  \right\|_{2} \geq \epsilon \right) \leq \epsilon
    \end{equation*}
\end{lemma}
This lemma is from \citet{boyd2006randomized}. We provide a proof for completeness.
\begin{proof}

    By applying Markov’s inequality
%
%    \begin{equation}
%        \Pr(|X| \geq a) \leq \frac{\E\left(|X|^n\right)}{a^n},
%
    we have
    \begin{align*}
        \Pr \left(\left \|  W_{t} \cdots  W_{s+1} e_v - \frac 1 N \bone  \right\|_{2}  \geq \epsilon \right)  \leq \frac{\E\left(\left \|  W_{t} \cdots  W_{s+1} e_v - \frac 1 N \bone  \right\|_{2}^2\right)}{\epsilon^2} \, .
    \end{align*}
   Denoting by $\tilde W_{k} =W_{k}- \frac 1 N \bone \bone^\top$,
   we need to prove that $\E\left(\left \|  W_{t} \cdots  W_{s+1} e_v - \frac 1 N \bone  \right\|_{2}^2\right) $ is bounded by $\lambda_2(\E[W^2])^{t-s}$.
  This is done by using \Cref{lem:lambW}.
\end{proof}

\subsection{Proof of \Cref{th:general_thm_hp}}
\gentheoremhp*
\begin{proof}
We have  \begin{align*}
        \Regbar_T  \leq \underbrace{ 3 \sum_{t=1}^{T} \sum_{u \in V} \frac{\Ind{u\in S_t}}{|S_t|}L \left \| x_t(u) - \bar x_t\right\|}_{(A)}
        +     \underbrace{\sum_{t=1}^{T} \sum_{v \in V} \frac{\Ind{v\in S_t}}{|S_t|} \vp{\nabla \loss_{t}(x_t(v),v)}{\bar x_t - x^{*}}}_{(B)}\, ,
    \end{align*}
    thanks to \Cref{lem:decomp}.
    Let us start by bounding Term $(B)$.
    \begin{align*}
    \sum_{t=1}^{T} \sum_{v \in V} \frac{\Ind{v\in S_t}}{|S_t|} \vp{\nabla \loss_{t}(x_t(v),v)}{{\bar x}_t - x^{*}} &\leq N \left( \frac{1}{N}\sum_{t=1}^{T} \sum_{v \in V} \Ind{v\in S_t} \vp{\nabla \loss_{t}(x_t(v),v)}{{\bar x}_t - x^{*}}\Pr\lr{S_t\neq \emptyset}\right). \nonumber
    \\& \leq N \left( \frac{\psi(x^*)}{\eta} + \frac{L^2}{\mu} \sum_{t=1}^T\eta \right)
\end{align*}
\begin{comment}Regarding Term $(A)$, the analysis is the same as for the case with known $|S_t|$.
Specifically, this term is bounded by $ 3\eta T N \frac{L^2}{\mu}\left (\frac{3\log{(\frac{NT^2}{\delta})}}{1- \lambda_2(\E[W^{\top}W])} + 2 + \frac{\delta}{NT}\right).$ 
\end{comment}
We then proceed by bounding Term $(A)$, by observing    
   \begin{align*} 3 \sum_{t=1}^{T} \sum_{u \in V} \frac{\Ind{u\in S_t}}{|S_t|}L \left \| x_t(u) - \bar x_t\right\|& \leq 3\eta \sum_{t=1}^{T} \sum_{u \in V} \max_{u \in V} L \left \| z_t(u) - {\bar z}_t\right\|  
    \end{align*}
Now, we focus on $\|z_{t+1}(v) -  \bar{z}_{t+1}\|_* $ . Starting from \Cref{eq:z_usual} and applying Lemma~\ref{cons}, we obtain
\begin{align}
    \label{eq: highgossip_1}
    \|z_{t+1}(v) -  \bar{z}_{t+1}\|_* &\leq  \sum_{s=1}^{t-1} \sqrt N L  \left\|  W_{t} \cdots  W_{s+1} e_v - \frac 1 N \bone  \right\|_{2} + 2L  \nonumber
    \\& \leq \sqrt NL(t-t^*) \epsilon + \sqrt NLt^* + 2NL \nonumber
    \\& \leq \sqrt NLT \epsilon + \sqrt NLt^* + 2NL
\end{align}
with probability at least $1 - \epsilon T$. Hence,
    \begin{align*}
    \Regbar_T 
     \leq 3\eta \sqrt N\frac{L^2}{\mu} T( t^* + \epsilon T + 2)  + N\left(\frac{D^2}{\eta} + \frac{L^2}{\mu} \sum_{t=1}^T\eta \right)
    &\leq N \frac{D^2}{\eta} +  \eta N\frac{L^2}{\mu} T( 3t^* + \epsilon T + 3) 
\end{align*}
with probability at least $1- \epsilon N T^2$. Setting $\epsilon = \frac{\delta}{NT^2} $ and $t^* = \frac{3\log{(\frac{NT^2}{\delta})}}{1- \rho^2}$, we have

\begin{equation}
    \Regbar_T \leq  N  \frac{D^2}{\eta} + 3\eta T N \frac{L^2}{\mu}\left (\frac{3\log{(\frac{NT^2}{\delta})}}{1- \rho^2}+ 3 + \frac{\delta}{NT}\right)\leq  N  \frac{D^2}{\eta} + 3\eta T N \frac{L^2}{\mu}\left (\frac{3\log{(\frac{NT^2}{\delta})}}{1- \rho^2}+ 4\right)
\end{equation}
with probability at least $1 - \delta$.
\end{proof}
\section{Proof of \Cref{th:lb}}
\thlower*
\begin{proof}
The proof is an adaptation of \citep[Theorem~3]{wang2022unified}.
Let $G$ be the cycle with $N = 4M$ nodes.
Suppose that for $3M$ nodes, the local loss functions are set to $0$ for all $t \in [T]$,
\[
    \ell_t(1,\cdot)= \cdots = \ell_t(3M,\cdot)= 0~.
\]
The local loss functions of the remaining nodes are reset every $M$ steps, in such a way that for each $k=0,\ldots,\lceil T/M \rceil -1$,
\[
    \ell_t(3M+1,\cdot)
= \cdots =
    \ell_t(4M,\cdot) = (N-M+1) H_{k}(\cdot) \qquad t \in \big[Mk+1, M(k+1)\big]
\]
where $H_0,\ldots,H_{\lceil T/M \rceil - 1}$ are random functions such that $H_k(x) = \varepsilon_k L \langle w,x\rangle$
and $\varepsilon_k$ are independent Rademacher random variables (equal to $1$ with probability $1/2$ and to $-1$ with probability $1/2$). The value $w$ is chosen as follows: pick $x_1,x_2\in\sX$ such that $\|x_1-x_2\|_2 = D$ and let $w = (x_1-x_2)/\|x_1-x_2\|_2$.
The network loss at time $t$ is thus 
\[
    \ellbar_t{(V,x)} = \frac{(N-M+1)M}{N}H_{ \lceil t/M  \rceil}(x)~.
\]
Observe that there is a delay introduced by communication over a cycle, such that agents $M+1 \ldots 2M$ have no access to information about the loss of any agent in $\{3M+1,\ldots , 4M\}$ before at least $M$ time steps have elapsed. 
Specifically, for any $v \in \{M+1,\ldots,2M\}$, predictions $x_{kt+1}(v),\ldots,x_{kt+M}(v)$ are computed without $v$ knowing $H_k$. Hence, using the standard lower bound for online learning \citep[Theorem~5.1]{orabona2019modern}, we have that in expectation with respect to $\varepsilon_0,\ldots,\ve_{\lceil T/M \rceil - 1}$, for all $v \in \{M+1,\ldots,2M\}$,
\begin{align*}
    \E&\left[\sum_{t=1}^T \ellbar_t\big(V,x_t(v)\big) - \min_{x \in \sX}\sum_{t=1}^T \ellbar_t(V,x)\right]
\\&=
    \E\left[\sum_{k=0}^{{\lceil T/M \rceil - 1}}\sum_{t=kM+1}^{(k+1)M}\frac{(N-M+1)}{N}H_k(x_t(v))
-
    \min\limits_{x\in \sX}\sum_{k=0}^{\lceil T/M \rceil - 1}\sum_{t=kM+1}^{(k+1)M}\frac{(N-M+1)}{N}H_k(x)\right]
\\&=
    \frac{(N-M+1)}{N}\E\left[\sum_{k=0}^{{\lceil T/M \rceil - 1}}\sum_{t=kM+1}^{(k+1)M} H_k(x_t(v))
-
    \min\limits_{x\in \sX}M\sum_{k=0}^{\lceil T/M \rceil - 1} H_k(x)\right]
\\&=
    \frac{(N-M+1)M}{N}\E\left[-\min\limits_{x\in\sX}\sum_{k=0}^{\lceil T/M \rceil - 1} H_k(x) \right]
\\&=
    \frac{(N-M+1)ML}{N}\E\left[\max\limits_{x\in\sX}\sum_{k=0}^{\lceil T/M \rceil - 1} \varepsilon_k \langle w,x\rangle \right]
\\&\ge
    \frac{(N-M+1)ML}{N}\E\left[\max\limits_{x\in\{x_1,x_2\}}\sum_{k=0}^{\lceil T/M \rceil - 1} \varepsilon_k \langle w,x\rangle \right]
\\&=
    \frac{(N-M+1)ML}{2N}\E\left[\left|\sum_{k=0}^{\lceil T/M \rceil - 1} \varepsilon_k \langle w,x_1-x_2\rangle\right| \right]
    \tag{using $\max\{a, b\} = \frac{1}{2}(a + b) + |a - b|$}
\\&=
    \frac{(N-M+1)MLD}{2N}\E\left[\left|\sum_{k=0}^{\lceil T/M \rceil - 1} \varepsilon_k \right| \right]
    \tag{using $w = (x_1-x_2)/\|x_1-x_2\|_2$}
\\&\ge
    \frac{(N-M+1)MLD}{2N}\sqrt{\frac{T}{M}} \tag{Khintchine inequality}
\\&=
    \frac{(N-M+1)LD}{2N}\sqrt{MT}~.
\end{align*}
Now, let $U = \{M+1, \ldots, 2M\} \subset V$. We have
\begin{align*}
    \E[\Regbar_T]
&=
    \frac{1}{N} \sum_{v\in V}\lr{ \sum_{t=1}^T\ellbar_t\big(V,x_t(v)\big)
-
    \min_{x \in \sX}\sum_{t=1}^T \ellbar_t(V,x)}
\\&\ge
    \frac{1}{N} \sum_{v\in U}\lr{ \sum_{t=1}^T\ellbar_t\big(V,x_t(v)\big)
-
    \min_{x \in \sX}\sum_{t=1}^T \ellbar_t(V,x)}
\\&\ge
    \frac{(N-M+1)DL}{2N}\sqrt{MT}
\end{align*}
Hence there exists a choice of $\varepsilon_1,\ldots,\varepsilon_{\lceil T/M \rceil - 1}$ such that 
\begin{align*}
     \Regbar_T\geq&\frac{(N-M+1)DL}{2N}\sqrt{ MT}~.
\end{align*}
Note that, on the $N$-cycle, the highest and smallest non-zero eigenvalues of the Laplacian are, respectively, $\lambda_1(G)=1$ and $\lambda_{N-1}(G)=2-2\cos(2 \pi /N)$ \cite[Chapter 6.5]{spielman2019spectral}. Using the inequality $1 - \cos(x) \ge x^2/5, \, \forall x \in [0,\pi]$ (recall that $N \ge 4$ implying $2\pi/N \le \pi$), we have $\lambda_{N-1}(G)\ge \frac {8 \pi^2}{5 N^2}$ and so $\sym(G) \le \frac{5N^2}{8 \pi^2}\le \frac{N^2}{8}$.  
Then, using \Cref{th:lam}, we have that $\rho/(1-\rho) = \sym(G) - 1 \le \frac{N^2}{8} - 1 \le \frac{N^2}{8}$. This in turn implies that, for all $0 \le \alpha \le 1$,
    \begin{align*}
        \Regbar_T
    \ge
        \frac{(N-M+1)DL}{2N}\sqrt{MT}
    \ge
        \frac {3DL}{16} \sqrt N \sqrt T 
    \ge
        \frac {3DL}{16} \left(\frac{8\rho}{1-\rho}\right)^{\frac{\alpha}{4}}N^{\frac{1-\alpha}{2}} \sqrt{T}
    %    \frac {3 \times \sqrt 2 DL}{8\sqrt 2} N^{1/4}\left({\frac{\rho}{1-\rho}}\right)^{1/8} \sqrt T
    %\ge
    %    \frac{3DL}{8} \left(\frac{\rho}{1-\rho}\right)^{\frac{1}{8}} N^{1/4}\sqrt{T}
    \end{align*}
    concluding the proof.
\end{proof}

\section{Spectral properties of the gossip matrix in the $p$-uniform case}
\subsection{Arbitrary graph.}
\thlam*

\begin{proof}
    Denoting by $L_1 = \textnormal{Lap}(\scG_1)$, one has:
    \begin{align*}
        \E[W_1^2] = \E[I - 2b L_1 + b^2 L_1^2]
        = I - 2b \E[L_1] + b^2 \E(L_1^2)
    \end{align*}
    \[
        \E[L_1] = p^2 D_{\scG} - p^2 A_{\scG} = p^2 L_{\scG}
    \]
    We compute:
    \[
        \E[L_1^2] = \E[D_{1}^2 - 2D_{1} A_{1} + A_{1}^2]
    \]
    where $D_1$ and $A_1$ denote the diagonal matrix of degrees and the adjacency matrix of $G_1$.
    \begin{align*}
        \E[(D_1^2)_{ii}]
        &= \E\left[\left(\sum_{j \in \mathcal{N}_i} \1(j \in S_1)\right)^2 \1(i \in S_1)\right]\\
        &= p \E \left[\sum_{j \in \mathcal{N}_i} \1(j \in S_1)\sum_{k \in \mathcal{N}_i} \1(k \in S_1) \right]\\
        &= p \E \left[\sum_{j \in \mathcal{N}_i}  \1(j \in S_1) \lr{\sum_{k \in \mathcal{N}_i, k\neq j } \1(k \in S_1) + \1(j \in S_1)} \right]\\
        &= p \left(p^2 D_{\scG} (D_{\scG} - I) + D_{\scG} p \right)_i\,,
    \end{align*}
    where $D_{\scG}$ denotes the matrix of degrees of the graph ${\scG}$.

    Finally,
    \[
        \E[(D_1^2)] = p^2 (p D_{\scG} (D_{\scG} - I) + D_{\scG})\,.
    \]
    where $A_{\scG}$ denotes the adjacency matrix of the graph ${\scG}$.
    Regarding $A_1^2$
    \begin{align*}
        \E[(A_1^2)_{ij}]
        &= \E\left[\sum_{k \in \mathcal{N}_i , k \in \mathcal{N}_j} \1(k \in S_1) \1(i \in S_1) \1(j \in S_1)\right]\\
        &=
        \begin{cases}
            p^3 \Big|\mathcal{N}_i \cap \mathcal{N}_j\Big| = p^3 (A_{\scG}^2)_{ij} & \text{ if } i \neq j\\
            \E\left[\sum_{k \in \mathcal{N}_i} \1(k \in S_1) \1(i \in S_1) \right]
            = p^2 (D_{\scG})_i & \text{ if } i = j
        \end{cases}
    \end{align*}
    Hence, finally,
    \[
        \E[A_1^2] = p^3 A_{\scG}^2 - p^3 D_{\scG} + p^2 D_{\scG}\,.
    \]
    Now, regarding $D_1 A_1$,
    \begin{align*}
        \E[(D_1 A_1)_{ij}] &= \E\left[ (D_1)_{ii} (A_1)_{ij}\right]\\
        & = \E\left[\1(i \in S_1) \1(j \in S_1) \1((i,j)\in E)\sum_{k \in \mathcal{N}_i  } \1(k \in S_1)\right]\\
        &=
        \begin{cases}
            p^3 (D_{\scG}A_{\scG})_{ij} - p^3 [A_{\scG}]_{ij}  + p^2 (A_{\scG})_{ij} & \text{ if }i \neq j\\
            0 & \text{ if } i = j
        \end{cases}
    \end{align*}
    Finally,
    \[
        \E[D_1A_1] = p^3 (D_{\scG} - I)A_{\scG} + p^2 A_{\scG}\,.
    \]
    Putting everything together,
    \begin{align*}
        \E[L_1^2] &= p^2 (p D_{\scG} (D_{\scG} - I) + D_{\scG}) + p^3 A_{\scG}^2 - p^3 D_{\scG} + p^2 D_{\scG}  - 2p^3 (D_{\scG}-I)A_{\scG} - 2 p^2 A_{\scG}\\
        &= p^3 D_{\scG}^2 - 2p^3 D_{\scG} + 2p^2 D_{\scG} - 2p^3 D_{\scG}A_{\scG} + 2 p^3 A_{\scG} - 2p^2 A_{\scG} + p^3 A_{\scG}^2\,.
    \end{align*}

    So in the end we have
    \begin{align*}
        \E[W_1^2] &= I
        - 2b p^2 (D_{\scG} - A_{\scG})
        + b^2 \left(p^3 (D_{\scG}^2 - 2D_{\scG} -2 D_{\scG}A_{\scG} +2 A_{\scG})+ p^3 A_{\scG}^2 + p^2 (2D_{\scG} - 2A_{\scG}) \right)\\
        &= I - 2bp^2(D_{\scG} - A_{\scG}) + b^2 [p^3 ((D_{\scG}-A_{\scG})^2 + 2 (A_{\scG}-D_{\scG})) + 2 p^2 (D_{\scG}-A_{\scG})\,.]
    \end{align*}
    \[ \boxed{ \E[W_1^2] = I - 2bp^2L_{\scG} + b^2 [p^3 (L_{\scG}^2 -2 L_{\scG}) + 2 p^2 L_{\scG}]\,.}\]

    So if $x$ is an eigenvector of $L_{\scG}$ with eigenvalue $\lambda_i(L_{\scG})$, then
    \begin{align*}
        \E[W_1^2]x&= x - 2bp^2 \lambda_i(L_{\scG}) x + b^2 [p^3 (\lambda_i(L_{\scG})^2 -2\lambda_i(L_{\scG}) )  + 2 p^2 \lambda_i(L_{\scG}) ] x\\
        &= \left(1 - 2bp^2 \lambda_i(L_{\scG}) + b^2 [p^3 (\lambda_i(L_{\scG})^2 -2\lambda_i(L_{\scG}) )  + 2 p^2 \lambda_i(L_{\scG}) ]\right) x\\
        &= \left(1 + ( - 2bp^2 + 2 b^2 p^2 - 2p^3 b^2) \lambda_i(L_{\scG}) + b^2 p^3 (\lambda_i(L_{\scG})^2  \right) x
    \end{align*}
    So that the eigenvalues of $\E[W_1^2]$ can easily be expressed as eigenvalues of $L_{\scG}$.
    It is easy to check that $f_{b,p} : x \mapsto 1 + ( - 2bp^2 + 2 b^2 p^2 - 2p^3 b^2) x + b^2 p^3 x^2 $ is a quadratic function decreasing on $]-\infty, 1 + (1-b)/ (bp) ]$.

    It is clear that
    \[
        1 + (1-b)/ (bp) = 1 + (1/b-1) /p \geq 1/bp \geq 1/b \geq 2\Delta({\scG})\geq \lambda_1(L_{\scG})\,.
    \]

    So that $f_{b,p}$ is decreasing on an interval containing all eigenvalues of the Laplacian $L_{\scG}$ and that
    $$
    \lambda_2(\E[W_1^2]) = f_{b,p}(\lambda_{N-1}(L_{\scG}))\,.
    $$
    Hence
    \[
        \boxed{\lambda_2(\E[W_1^2]) = 1 + ( - 2bp^2 + 2 b^2 p^2 - 2p^3 b^2) \lambda_{N-1}(L_{\scG}) + b^2 p^3 (\lambda_{N-1}(L_{\scG}))^2\,.}
    \]
%
    \begin{comment}
        Setting $b = \frac{1}{\lambda_1(\textnormal{Lap}({\scG}))}$.
        \begin{align*}
            \lambda_2(\E[W_1^2])
            &\leq  1 + ( - 2bp^2 + 2 b^2 p^2) \lambda_{N-1}(L_{\scG}) + b^2 p^3 (\lambda_{N-1}(L_{\scG}))^2\\
            &\leq  1 - 2\frac{\lambdanminone}{\lambdaone}p^2 + 2 \frac{\lambdanminone}{\lambdaone^2} p^2 + \frac{\lambdanminone^2}{\lambdaone^2} p^3 \\
            &\leq  1 - \frac{\lambdanminone}{\lambdaone}p^2 \lr{ 2- \frac{2}{d_{\min}+1}} + \frac{\lambdanminone^2}{\lambdaone^2} p^3 \\
            &\leq  1 - \frac{\lambdanminone}{\lambdaone}p^2 + \frac{\lambdanminone^2}{\lambdaone^2} p^3 ,
        \end{align*}
        which yields the second result.
    \end{comment}

Setting $b =1/\lambda_1(G)$ and rewriting yields
 \begin{align*}
        \rho^2 = \lambda_2(\E[W_1^2])= 1 - \frac{2p^2}{\sym(G)}\left(1 - \frac{1-p}{\lambda_1(G)} - \frac{p}{2\sym(G)}\right)~.
    \end{align*}
\end{proof}

\subsection{Special cases}

\paragraph{Clique.}

When ${\scG}$ is the clique, and $b = 1/(N)$,    since $\lambdanminone = N$
\[\lambda_2(\E[W_1^2]) = 1 - p^2 - \frac 1 {N} ((N-2) p^3 + 2 p^2)\,,\]
using  \Cref{th:lam}.

\paragraph{Strongly Regular graphs.}

Strongly Regular graphs are such that
\begin{itemize}
    \item they are $k$-regular, for some integer $k$
    \item there exists an integer $m$ such that for every pair of vertices $u$ and $v$ that are neighbors in ${\scG}$, there are $m$ vertices that are neighbors of both $u$ and $v$
    \item there exists an integer $n$ such that for every pair of vertices $u$ and $v$ that are not neighbors in ${\scG}$, there are $n$ vertices that are neighbors of both $u$ and $v$.
\end{itemize}
Such graphs' adjacency matrices have eigenvalues $k$ with multiplicity $1$
and $r$ and $s$ defined as follows:

$r = \frac{m-n + \sqrt{(m-n)^2+ 4 (k-n)}}{2}$ and \\ $s = \frac{m-n - \sqrt{(m-n)^2+ 4 (k-n)}}{2}\,.$

Hence, their Laplacian have eigenvalues $0$ with multiplicity $1$ and and $k-r$ and $k-s$.
This yields
$\lambdaone = k-s$ and $\lambdanminone = k-r$.

Using \Cref{th:labgen},  we have that if $b = 1/ \lambda_1(\lap({\scG}))$, then
\begin{equation*}
    \rho^2 = \lambda_2(\E[W_1^2]) \leq  1 -p^2  \frac{k-r}{k-s}
\end{equation*}
In particular, when ${\scG}$ is the lattice graph, with $N= M^2$ vertices
\begin{equation*}
    \rho^2 = \lambda_2(\E[W_1^2]) \leq  1 -\frac 1 2 p^2  \,.
\end{equation*}

replacing $k$ by $2M-2$, $m$ by $M-2$ and $n$ by $2$.

\paragraph{Grid.} Consider ${\scG}$  a grid of dimension $2$, with $N= M^2$.
${\scG}$ is the product of two paths graphs of length $M$.
Then if $\mu_1\ldots \mu_M$ are the eigenvalues of the path graph of length $M$, then all eigenvalues of $\lap(G)$ can be rewritten as  $\mu_i + \mu_j$ for some $i$ and $j$---see \citep[Theorem 3]{barik2015laplacian}.
Furthermore we know that $\mu_i = 2 (1 - \cos(\pi(M-i)/M))$---see, e.g., \citep[Theorem 6.6]{spielman2019spectral}---so that $\lambdaone = 2 \mu_1 = 4-4\cos(\pi (M-1)/M$  and $\lambdanminone = \mu_N + \mu_{N-1} = 2-2\cos(\pi/M) $.

Hence setting $b = 1/ \lambda_1(\lap({\scG}))$
\begin{align*}
    \rho^2 = \lambda_2(\E[W_1^2]) \leq &
    1 -p^2 \frac{2-2\cos(\pi/M)}{4-4\cos(\pi (M-1)/M)}
\end{align*}
by using \Cref{th:labgen}.

\section{Proof of \Cref{erdosrenyi}}
\erdosrenyi*
\begin{proof}
As bound~\eqref{eq:second-bound} in \Cref{th:lin_thm} applies, we just have to compute the spectral gap.
    
    Regarding the expression of $\rho^2 = \lambda_2(\E[W_1^2])$, we take the same steps as for the proof of \Cref{th:lam}.
    
    We observe that
        \begin{align*}
            \E[W_1^2] = \E[I - 2b L_1 + b^2 L_1^2]
            = I - 2b \E[L_1] + b^2 \E(L_1^2)\,.
        \end{align*}
    We compute
        \[
            \E[L_1] = p^2 D_{\scG} - p^2 A_{\scG} = p^2 L_{\scG}
        \]
    and
        \[
            \E[L_1^2] = \E[D_{1}^2 - 2D_{1} A_{1} + A_{1}^2]\,.
        \]

    Mutatis mutandis in the computations of the proof of \Cref{th:lam}, we get the following equalities for the first term,
        \begin{align*}
            \E[(D_1^2)_{ii}]
            &= \E\left[\left(\sum_{j \in \mathcal{N}_i} \1((i,j) \in E_1)\right)^2 \right]\\
            &= \E \left[\sum_{j \in \mathcal{N}_i}  \1((i,j) \in E_1) \lr{\sum_{k \in \mathcal{N}_i, k\neq j } \1((i,k) \in E_1) + \1((i,j) \in E_1)} \right]\\
            &= \E \left[
            \E\left[
            \sum_{j \in \mathcal{N}_i}  \1((i,j) \in E_1) \lr{\sum_{k \in \mathcal{N}_i, k\neq j } \E \left[
            \1((i,k) \in E_1)
            +\1((i,j) \in E_1) \Big| (i,k) \in S_1
            \right]} 
            \Big| (i,j) \in S_1
            \right]
            \right]\\
            &= p^2 q \left(p q D_{\scG} (D_{\scG} - I) + D_{\scG}  \right)_i\,,
        \end{align*}
    the second term,
        \begin{align*}
            \E[(A_1^2)_{ij}]
            &= \E\left[\sum_{k \in \mathcal{N}_i , k \in \mathcal{N}_j} \1((k,i) \in E_1)  \1((k,j) \in E_1)\right]\\
            &=
            \begin{cases}
                p^3 \left|\mathcal{N}_i \cap \mathcal{N}_j\right| = p^3 q^2(A_{\scG}^2)_{ij} & \text{ if } i \neq j\\
                 p^2 \sum_{k \in \mathcal{N}_i}\E\left[\1((i,k)\in E_1)\Big | \1((i,k) \in S_1) \right] 
                = p^2 q (D_{\scG})_i & \text{ if } i = j
            \end{cases}
        \end{align*}
    and the third term.
 \begin{align*}
        \E[(D_1 A_1)_{ij}] &= \E\left[ (D_1)_{ii} (A_1)_{ij}\right]\\
        & = \E\left[ \1((i,j)\in E_t)\sum_{k \in \mathcal{N}_i  } \1((i,k)\in E_1)\right]\\
        &=
        \begin{cases}
            p^3 q^2(D_{\scG}A_{\scG})_{ij} - p^3q^2 [A_{\scG}]_{ij}  + p^2 q (A_{\scG})_{ij} & \text{ if }i \neq j\\
            0 & \text{ if } i = j
        \end{cases}
    \end{align*}
    Finally, by adding these three inequalities and rearranging,
        \[ \boxed{ \E[W_1^2] = I - 2bp^2 qL_{\scG} + b^2 \left(p^3 q^2 (L_{\scG}^2 -2 L_{\scG}) + 2 p^2 q L_{\scG}\right)\,.}\]
     It is easy to check that $f_{b,p,q} : x \mapsto 1 + ( - 2bp^2 q + 2 b^2 p^2 q - 2p^3 b^2 q^2) x + b^2 p^3 q^2 x^2$ is a quadratic function decreasing on $(-\infty, 1 + (1-b)/ (bpq) ]$.

    Again,
    \[
        1 + (1-b)/ (bpq) = 1 + (1/b-1) /pq \geq 1/bpq \geq 1/b \geq 2\Delta({\scG})\geq \lambda_1(L_{\scG})\,.
    \]
    so that $f_{b,p,q}$ is decreasing on an interval containing all eigenvalues of the Laplacian $L_{\scG}$ and that
    $$
    \rho^2 = \lambda_2(\E[W_1^2]) = f_{b,p,q}(\lambda_{N-1}(L_{\scG}))
    $$
concluding the proof.
    \end{proof}
\section{Some results mentioned in the discussions}

\subsection{Tuning the learning rate in the general case (Proof of \Cref{eq:rate_diffp_2})}

\begin{cor}
    Assume $N\ge 2$.
    Assume each agent runs an instance of \algoname\ with learning rate $\eta=
    \frac{D \sqrt{\mu} \pmin} {L\sqrt{ T}}$ and gossip matrices set according to \eqref{eq:def_w}.
% and i.i.d.\ gossip matrices $W_1,W_2,\ldots$.
Then, the expected network regret can be bounded by
\begin{equation}
    \expRegbar
\le
    12 DL\frac{\sym(G)}{\pmin}N\sqrt{\frac{T}{\mu}}~.
\label{eq:rate_diffp}
\end{equation}
If 
$\pmin^2\pbar\leq \frac{1}{\sqrt N}$,
and each agent runs an instance of \algoname\ with learning rate
$
    \eta
=
    \frac{D \sqrt{\mu} \pmin} {L\sqrt{ T}}
$,
which only requires knowing $\pmin$, and gossip matrices set according to \eqref{eq:def_w} then
\begin{equation}
    \expRegbar
\le
    12 DL\frac{\sym(G)}{\pmin}N\sqrt{\frac{T}{\mu}}~.
%\label{eq:rate_diffp}
\end{equation}
\end{cor}
\begin{proof}
Thanks to \Cref{th:lin_thm} , we have  \[ \expRegbar
    \le
        \frac{N D^2 }{\eta}
        +
        \frac{L^2}{\mu} \eta \bigg({\bar p N + \bar p (1-\bar p) - \sigma_p^2 }
    + 6 +
        3 \bar p N \frac{\rho}{1-\rho} \bigg)T\,.\] 
          We note that $\bar p N + \bar p (1-\bar p) - \sigma_p^2 \le  \bar 2 p N$ and 
        \[\frac{\rho}{1-\rho}\le \frac{1}{1-\rho}\le \frac{1}{1-\sqrt{1-\frac{\pmin^2}{\kappa}}} \le  2 \frac{\kappa}{\pmin^2} \, ,\]
        where the second inequality follows from \Cref{th:labgen}, and the last one follows from the concavity of $\sqrt \cdot$. 
        
        Consequently, we can prove 
        \[ \expRegbar
    \le
        \frac{N D^2 }{\eta}
    + \bigg(6 +
        8 \bar p N \frac{\kappa}{\pmin^2} \bigg)T \le
        \frac{N D^2 }{\eta}
    + \bigg(11\bar p N \frac{\kappa}{\pmin^2} \bigg)T\,.\]
Setting
$
    \eta
=
    \frac{D \sqrt{\mu} \pmin} {L\sqrt{ T}}
$,
which only requires knowing $\pmin$, suffices to obtain the bound
\begin{equation}
    \expRegbar
\le
    12 DL\frac{\sym(G)}{\pmin}N\sqrt{\frac{T}{\mu}}~.
%\label{eq:rate_diffp}
\end{equation}
If additionally
$\pmin^2\pbar\leq \frac{1}{\sqrt N}$, we have 
$\bar p N \le  \sqrt N \frac{\kappa}{\pmin^2} $,  so that using \Cref{th:lin_thm} along with  $\frac{\rho}{1-\rho} \le  2\frac{\kappa}{\pmin^2}$  yields
\[ \expRegbar
    \le
         \frac{N D^2 }{\eta}
        +
        2 \sqrt N \frac{\kappa}{\pmin^2}
    + 6 +
         3  \sqrt N \frac{\kappa}{\pmin^2}
        \le \frac{N D^2 }{\eta}
        +
        11 \sqrt N \frac{\kappa}{\pmin^2}\,.\]

In turn, setting $
    \eta
=
    \frac{N^{1/4} D \sqrt{\mu} \pmin} {L\sqrt{ T}}
$ yields
\begin{equation*}
    \expRegbar
\le
    12 DL\frac{\sym(G)}{\pmin}N^{3/4}\sqrt{\frac{T}{\mu}}~.
\end{equation*}
\end{proof}

\subsection{Regret bounds with known $|S_t|$}
We study the variant of the algorithm where $\nabla \ell_t(v)$ in \Cref{eq:gr} is replaced by $\frac N {|S_t|} \nabla \ell_t(v)$.
Using this variant of the algorithm yields the bound given by the following theorem.
\begin{restatable}{theorem}{gentheorem}
    \label{th:generaL_1hm}
    With any $\eta>0$, the network regret can be bounded by
    \begin{align*}
        \expRegbar& \leq \frac{D^2 }{\eta} + \frac{L^2}{\mu} T \eta + 3 \eta  N^{} \frac{L^2}{\mu}   \frac{1}{1 - \rho}  T\,,
    \end{align*}
    Setting
    $\eta = D \sqrt{\mu}/(L\sqrt N \sqrt{T})$,
    \[\expRegbar \leq  4 \frac{LD}{\sqrt{\mu}} \sqrt{N} \frac{1}{1 - \rho} \sqrt{T} \,.\]
\end{restatable}
\begin{proof}
    Like for \Cref{th:lin_thm}, the proof relies on the use of an omniscient agent knowing the gradients of all losses  incurred  up to time $t-1$.

    Let us define these quantities $\barz_t$ and $\bar{g}_t$
    \begin{align*}
        &\bar{g}_t = \frac{1}{N}\sum_{v \in V} g_t(x_t(v), v)
        = \sum_{v \in V} \frac{\mathds{1}(v \in S_t)}{|S_t|} \nabla \loss_{t}(x_t(v),v)
        \\
        &\barz_t = \frac{1}{N} \sum_{v \in V} z_t(v).
    \end{align*}

    The decision of the omniscient agent is defined as
    \begin{equation*}
        \bar{x}_t = \operatorname{argmin}_{x \in X}\left\{\langle \bar{z}_t, x\rangle+\frac{1}{\eta} \psi(x)\right\},
    \end{equation*}
as in the proof of \Cref{th:lin_thm}.
    We still have
    \begin{equation*}
        \barz_{t+1} =  \barz_{t} + \bar{g}_t.
    \end{equation*}

    Still like for \Cref{th:lin_thm}, the proof of the theorem relies on the use of  \Cref{lem:decomp}, where $y_t$ is set to $\bar x_t$.

    \begin{equation}
    \label{eq:regret_decomposition_1_B}
        \Regbar_T\leq \underbrace{ 3 \sum_{t=1}^{T} \sum_{u \in V} \frac{\Ind{u\in S_t}}{|S_t|}L \left \| x_t(u) - {\bar x}_t\right\|}_{(A)}
        +
        \underbrace{\sum_{t=1}^{T} \sum_{v \in V} \frac{\Ind{v\in S_t}}{|S_t|} \vp{\nabla \loss_{t}(x_t(v),v)}{{\bar x}_t - x^{*}}}_{(B)}\, .
    \end{equation}

    Let us focus on Term $(B)$ first.
    According to the usual analysis of FTRL, we have
    \begin{equation}
        \label{ptB_B}
        \sum_{t=1}^{T} \sum_{v \in V} \frac{\Ind{v\in S_t}}{|S_t|} \vp{\nabla \loss_{t}(x_t(v),v)}{\bar x_t - x^{*}} \leq \frac{\psi(x^*)}{\eta} + \frac{L^2}{\mu} \sum_{t=1}^T\eta\1 \lr{S_t \neq \emptyset}\,.
    \end{equation}
    since the omniscient agent's updates correspond to FTRL on linear losses equal to  $\langle \tilde L_1, x\rangle$
    where
    $\tilde{l}_t =\sum_v \nabla \ell_1(x_t(v),v)  \frac{\mathds{1}(v\in S_t)}{|S_t|}$.

    Regarding Term $(A)$, s
    we still have (like  in \Cref{eq:z})
    \begin{equation}
        \label{eq:z_B}
        \left\|x_t(v)-\bar{x}_t\right\| \leq \eta/\mu  \left\|z_t(v)-\bar{z}_t\right\|_{*}.
    \end{equation}
    For any $t \in [T]$ and any $v \in [N]$, it also still holds that
    \begin{align*}
        \bz_{t+1}  
        &= \sum_{s=1}^{t-1} W_t \cdots W_{s+1} \bg_s + \bg_t\,
    \end{align*}
and that
    \begin{align*}
        \bar{z}_{t+1} = \frac 1 N \sum_{s=1}^t \bone^\top \bg_s.
    \end{align*}
     Consequently, we can still prove
    \begin{align*}
        z_{t+1}(v) -  \bar{z}_{t+1} 
        & =  \sum_{s=1}^{t-1} \bg_s^T \left[ W_t \cdots  W_{s+1} - \frac 1 N \bone \bone^\top \right] ^T e_v + \bg_t^\top (I - \frac 1 N \bone \bone^\top) e_v,
    \end{align*}
    and compute :
    \begin{align*}
        \left \|z_{t+1}(v) - \barz_{t+1} \right \|_* 
        &
        \leq \left \|\sum_{s=0}^{t-1} \left(  \sum_{u=1}^N \left ([W_t \cdots  W_{s+1}]_{u,v} -  \frac{1}{N} \right) g_{s}(x_{s}(u), u)\right)\right\|_* + \left\| g_{t}(x_t(v), v) - \bar{g}_{t}\right\|_*
        \\& \leq \sum_{s=0}^{t-1} \sum_{u=1}^N  \left|[W_t \cdots  W_{s+1}]_{u,v} -  \frac{1}{N} \right| \left\| g_{s}(x_{s}(u), u)\right\|_* + \left\| g_{t}(x_t(v), v) - \bar{g}_{t}\right\|_*
        \\& \leq \sum_{s=0}^{t-1} \sum_{u=1}^N \left \|W_t \cdots  W_{s+1}e_v-  \frac{\bone}{N} \right\|_{\infty} \left\| g_{s}(x_{s}(u), u)\right\|_* + \sum_{u=1}^N \left \| I e_v-  \frac{\bone}{N} \right\|_{\infty} \left\| g_{s}(x_{s}(u), u)\right\|_* 
    \end{align*}
    Since $\| \nabla \ell_1( \cdot, v) \|_{*}\leq L$ for all $v$,
    \begin{align}
        \|z_{t+1}(v) -  \bar{z}_{t+1}\|_*
        & \leq  \sum_{s=1}^{t-1}  |S_s| \left(\frac{N}{|S_s|}L\right) \left\|  W_t \cdots  W_{s+1} e_v - \frac 1 N \bone  \right\|_{\infty} +  NL\nonumber\\
        & \leq  \sum_{s=1}^{t-1} N L  \left\|  W_{t} \cdots  W_{s+1} e_v - \frac 1 N \bone  \right\|_{\infty} + NL\nonumber\\
        & \leq  \sum_{s=1}^{t-1} N L  \left\|  W_{t} \cdots  W_{s+1} e_v - \frac 1 N \bone  \right\|_{2} + NL \,, \label{eq: boundB_B}
    \end{align}
    considering $\|\cdot\|_{\infty} \leq  \|\cdot\|_2$.
By using Jensen's inequality, 
\[
            \mathbb{E}\left[\left\|W_{t} \cdots  W_{s+1} e_v - \frac \bone N \right\|_2 \right] \leq \sqrt{\mathbb{E}\left[(W_{t} \cdots  W_{s+1} e_v - \frac \bone N )^T (W_{t} \cdots  W_{s+1} e_v - \frac 1 N)  \right]} \nonumber\\
\]
By using \Cref{lem:lambW}, we obtain
    \begin{align*}
        \E[\|z_{t+1}(v) -  \bar{z}_{t+1}\|_*]
        & \leq   \sum_{s=1}^{t} N^{} L  \lambdatwo^{\frac{t-s}{2}} \\
        & \leq N^{}L \frac{1}{1 - \rho}
    \end{align*}
% &
% So that using this
% &
    Hence, coming back to the expression of Term $(A)$, and following \Cref{eq:z}
    \begin{align*}
        \sum_{t=1}^T \sum_{u \in \mathcal{V}} \frac{1}{|\mathcal{S}_t|} \mathds{1}(u \in \mathcal{S}_t) L \|x_t(u) - \bar{x}_t\|
        &
        \leq \sum_{u \in \mathcal{V}} \sum_{t=1}^T \frac{1}{|\mathcal{S}_t|} \eta/\mu L \|z_t(u) - \bar{z}_t\| \mathds{1}(u \in \mathcal{S}_t)
    \end{align*}
    and combining with the following,
    \begin{align*}
        &\E \left[\sum_{u \in \mathcal{V}} \sum_{t=1}^{T} \frac{1}{|\mathcal{S}_t|} \|z_{t} (u) - \bar{z}_{t} \|_* \mathbf{1}(u \in \mathcal{S}_{t}) \right]\\
        &= \sum_{u \in \mathcal{V}} \sum_{t=1}^{T}  \E\left[   \frac{\mathds{1}(u \in \mathcal{S}_{t}}{|\mathcal{S}_t|}\right]  \E\left[\|z_{t} (u) - \bar{z}_{t} \|_* ) \right]\\
        &\leq \sum_{u \in \mathcal{V}}\sum_{t=1}^{T} \E\left[   \frac{\mathds{1}(u \in \mathcal{S}_{t})}{|\mathcal{S}_t|}\right]   \times N L \frac{1}{1 - \rho}
        \\
        &\leq \sum_{t=1}^{T}    \lr{1- \Pi_{v \in \mathcal V} \lr{1-p_v}}N L \frac{1}{1 - \rho}
%\\
%&\leq \sum_{t=1}^{T}    \lr{1- \lr{1-p_{\max}}^N}N L \frac{\sqrt{\lambda_2}}{1 - \sqrt{\lambda_2}}
%\\
%&\leq \sum_{t=1}^{T}    \min(1, N p_{\max})N L,  \frac{\sqrt{\lambda_2}}{1 - \sqrt{\lambda_2}}
        \,,
    \end{align*}
    we have
    \begin{equation}
        \label{eq:ptA_B}
        \E\left[\sum_{t=1}^T \sum_{u \in \mathcal{V}} \frac{1}{|\mathcal{S}_t|} \mathds{1}(u \in \mathcal{S}_t) L \|x_t(u) - \bar{x}_t\|\right]
        \leq \eta    \frac{L^2}{\mu} N \frac{1}{1 - \rho} \tilde T
    \end{equation}
    where $\tilde T = \lr{1- \Pi_{v \in \mathcal V} \lr{1-p_v}}T$ is the expected number of times where there is at least one active agent.

    Recalling \Cref{eq:regret_decomposition_1_B} and combining \cref{eq:ptA_B} and \Cref{ptB_B}, we obtain
    \begin{align*}
        \mathbb{E} (\Regbar_T)& \leq \frac{D^2 }{\eta} + \frac{L^2}{\mu} \tilde T \eta + 3 \eta  N  \frac{L^2}{\mu}   \frac{1}{1 - \rho} \tilde T \,.
    \end{align*}
This  directly yields
    \begin{align*}
        \mathbb{E} (\Regbar_T)& \leq \frac{D^2 }{\eta} + \frac{L^2}{\mu} T \eta + 3 \eta  N  \frac{L^2}{\mu}   \frac{1}{1 - \rho}  T \,,
    \end{align*}
and the computation with the appropriate $\eta$ follows.

\end{proof}

\begin{restatable}{theorem}
{gentheoremhpst}
    \label{th:general_thm_hp_st}
    (Known $|S_t|$) For any $\eta>0$, with probability $ 1- \delta$, the network regret can be bounded by
    \begin{equation*}
        \Regbar_T \leq \frac{D^2 }{\eta} + \frac{L^2}{\mu} T \eta + 3\eta T N \frac{L^2}{\mu}\left (\frac{9\log{\frac{NT^2}{\delta}}}{1- \lambdatwo} + 2 + \frac{1}{T^2}\right).
    \end{equation*}
\end{restatable}
% We have 
% \begin{equation*}
%     \left\|z_t(v)-\bar{z}_t\right\|_{*} \leq NL \frac{\log{T^2N}}{1-\lambda_2}+2L
% \end{equation*}
% with probability at least $1 - \frac{1}{T^2N}$.

% When $t-s+1  \geq \frac{3\log{\epsilon^{-1}}}{\log{\lambda_2^{-1}}} = t^{*}$, we 
% \begin{equation*}
%         \Pr(||W^{s,t} e_v - \frac{\bone}{N}||_{2} \geq \epsilon ) \leq \epsilon
% \end{equation*}
%
%Note that $\left \|W^{k+1,t-1} e_v - \frac{\bone}{N} \right\|_{1} \leq 1$.

%By applying Boole's inequality and De Morgan's laws,
%
% \begin{equation}
%     \sum_{t=1}^{T}\left\|z_t(v)-\bar{z}_t\right\|_* \leq NLt^*T +  NLT^2\epsilon + 2NLT.
% \end{equation}
We have  \begin{align*}
        \Regbar_T  \leq \underbrace{ 3 \sum_{t=1}^{T} \sum_{u \in V} \frac{\Ind{u\in S_t}}{|S_t|}L \left \| x_t(u) - \bar x_t\right\|}_{(A)}
        +     \underbrace{\sum_{t=1}^{T} \sum_{v \in V} \frac{\Ind{v\in S_t}}{|S_t|} \vp{\nabla \loss_{t}(x_t(v),v)}{\bar x_t - x^{*}}}_{(B)}\, ,
    \end{align*}
    thanks to \Cref{lem:decomp}.
Thanks to \Cref{eq: boundB_B} , we can bound Term $(B)$ in the following way
\begin{align*}
     \sum_{v \in V} \frac{\Ind{v\in S_t}}{|S_t|} \vp{\nabla \loss_{t}(x_t(v),v)}{\bar x_t - x^{*}} \frac{\psi(x^*)}{\eta} + \frac{L^2}{\mu} \sum_{t=1}^T\eta
    \end{align*}
We then proceed by bounding Term $(A)$, by observing    
   \begin{align*} 3 \sum_{t=1}^{T} \sum_{u \in V} \frac{\Ind{u\in S_t}}{|S_t|}L \left \| x_t(u) - \bar x_t\right\|& \leq 3\eta \sum_{t=1}^{T} \sum_{u \in V} \max_{u \in V} L \left \| z_t(u) - {\bar z}_t\right\|  
    \end{align*}
Now, we focus on $\|z_{t+1}(v) -  \bar{z}_{t+1}\|_* $ . Starting from \Cref{eq: boundB_B} and applying \Cref{cons}, we obtain
\begin{align}
    \label{eq: highgossip}
    \|z_{t+1}(v) -  \bar{z}_{t+1}\|_* &\leq  \sum_{s=1}^{t-1} N L  \left\|  W_{t} \cdots  W_{s+1} e_v - \frac 1 N \bone  \right\|_{2} + 2NL  \nonumber
    \\& \leq NL(t-t^*) \epsilon + NLt^* + 2NL \nonumber
    \\& \leq NLT \epsilon + NLt^* + 2NL
\end{align}
with probability at least $1 - \epsilon T$. Hence,
    \begin{align*}
    \Regbar_T 
    & \leq 3\eta N\frac{L^2}{\mu} T( t^* + \epsilon T + 2)  + \frac{\psi(x^*)}{\eta} + \frac{L^2}{\mu} \sum_{t=1}^T\eta
\end{align*}
with probability at least $1- \epsilon N T^2$. Set $\epsilon = \frac{\delta}{NT^2} $ and $t^* = \frac{3\log{(\frac{NT^2}{\delta})}}{1- \lambda_2(\E[W^{\top}W])}$, we have
\begin{equation*}
    \Regbar_T \leq \frac{D^2 }{\eta} + \frac{L^2}{\mu} T \eta + 3\eta T N \frac{L^2}{\mu}\left (\frac{3\log{(\frac{NT^2}{\delta})}}{1- \lambda_2(\E[W^{\top}W])} + 2 + \frac{\delta}{NT}\right),
\end{equation*}
with probability at least $1 - \delta$.

\subsection{Nonstationary case}
\begin{cor}Suppose that the activation probabilities change over time within known bounds $\pmin$ and $\pmax$, and the activations events are independent.
Assume each agent runs an instance of \algoname\ with learning rate $\eta > 0$ and gossip matrices set according to \Cref{eq:def_w}.
% and i.i.d.\ gossip matrices $W_1,W_2,\ldots$.
Then, the expected network regret can be bounded by
    \begin{align}
    \nonumber
        \expRegbar
    &\le
        \frac{N D^2 }{\eta}
        +
        \frac{L^2}{\mu} \eta \bigg({\bar p N + \bar p (1-\bar p) - \sigma_p^2 } + 6 +
        3 \min\big(\bar p N, \sqrt N\big) \frac{\sqrt{B_{\lambda}}}{1-\sqrt{B_{\lambda}}} \bigg)T
    \end{align}
    where we denote by $B_{\lambda}= 1 - \frac {(p^{\min})^2}{\sym}  $.
    \end{cor}
\begin{proof}
The proof works exactly like that of \Cref{th:lin_thm}, with the difference that \Cref{eq:induc} does not hold anymore. Instead, we replace it by
\begin{align*}
    \mathbb{E}\left[\left\|W_{t} \cdots  W_{s+1} e_v - \frac \bone N \right\|_2 \right]
    \leq B_{\lambda}^{\frac{t-s}{2}} \,.
\end{align*}
 In fact, thanks to \Cref{th:labgen}, we know that $\lambda_2(\E[W_k^2]) \leq B_{\lambda}$ for all $k\in [T]$.
To prove the above, we then proceed by induction.
\begin{align*}
    \E\left[\left\|W_{k} \cdots  W_{s+1} e_v - \frac \bone N \right\|_2^2\right]
    &=\E\left[e_v^T \tilde W_{s+1} \cdots \tilde W_{k-1} \tilde W_{k}^2 \tilde  W_{k-1} \cdots \tilde W_{s+1} e_v  \right]\\
    &=\E\left[e_v^T \tilde W_{s+1} \cdots \tilde W_{k-1}  \E[\tilde W_{k}^2| \mathcal{F}_{k-1}] \tilde  W_{k-1} \cdots  \tilde W_{s+1} e_v  \right]\\
    & \leq   \left\| \mathbb{E}[W_k^2] - \frac 1 N \bone \bone^\top \right\|_{\textnormal{op}}  B_{\lambda}^{k-1-s}\\
    & \leq \sqrt{\lambda_2(\E[W_k^2])}B_{\lambda}^{k-1-s}\\
    &\leq B_{\lambda}^{k-s}\,,
\end{align*}
where $\tilde W_{s}$ is $W_s - \frac 1 N \bone \bone^T.$

Hence,
\[\E\left[\left\|W_{k} \cdots  W_{s+1} e_v - \frac \bone N \right\|_2\right]
\leq B_{\lambda}^{\frac{k-s}{2}}\,.\]
The rest of the proof is identical to that of  \Cref{th:lin_thm}.
\end{proof}
\subsection{Bound on $\rho/(1-\rho)$}

\begin{lemma} In the $p$-uniform case, if the gossip matrices are set according to \Cref{eq:def_w}, then
    \[
    \frac{\rho}{1-\rho} \le \frac{2\sym(G)}{p^2}
\]
 for all $0\le p\le1$.
\end{lemma}
\begin{proof}
%    In reality, we prove a stronger inequality:
%   \[ \frac{\rho}{1-\rho} \le \frac{2\sym(G)}{p^2},\]
%   for all $0\le p\le1$.
We have
\begin{equation*}
     \frac{\rho}{1-\rho} \le \frac{1}{1-\rho}  \le \frac{1}{1-\sqrt{1-\frac{p^2}{\sym(G)}}} 
\end{equation*}
thanks to \Cref{th:labgen}. We also have $1-\sqrt{1-y}\ge \frac 1 2 y$ for all $y\le 1$ by concavity of $\sqrt \cdot$. Hence, 
\begin{equation*}
     \frac{\rho}{1-\rho}  \le \frac{2\sym(G)}{p^2} .
\end{equation*}
Since $\frac{\lambda_1(G)}{\lambda_1(G)-1} \ge 1$, this also yields $\frac{\rho}{1-\rho} \le \frac{2\sym(G)}{p^2}\frac{\lambda_1(G)}{\lambda_1(G)-1}\,$.

\end{proof}

\section{Experimental details}
The code is available at https://anonymous.4open.science/r/DistOLR-6085/.